\documentclass{article}

\usepackage[preprint,nonatbib]{neurips_2022}

\usepackage[utf8]{inputenc} 
\usepackage[T1]{fontenc}    
\usepackage{hyperref}       
\usepackage{url}            
\usepackage{booktabs}       
\usepackage{amsfonts}       
\usepackage{nicefrac}       
\usepackage{microtype}      
\usepackage{xcolor}         
\usepackage{amsmath}
\usepackage{amsthm}
\usepackage{dsfont}
\usepackage{amsfonts}
\usepackage{bbold}
\usepackage{dirtytalk}

\usepackage{graphicx,tikz}
\usetikzlibrary{fit,matrix,positioning,calc}

\usepackage{authblk}

\newtheorem{prop}{Proposition}

\newtheorem{theorem}{Theorem}
\newtheorem{lemma}{Lemma}
\newtheorem{coro}{Corollary}
\newtheorem{remark}{Remark}

\DeclareMathOperator{\fat}{fat}
\DeclareMathOperator{\Pdim}{Pdim}
\DeclareMathOperator{\VCdim}{VCdim}

\DeclareMathOperator*{\esssup}{ess\,sup}
\newcommand{\dd}{\mathrm{d}}
\usepackage[normalem]{ulem}

\newcommand {\RR} {\mathbb R}

\newcommand {\NN} {\mathbb N}
\newcommand {\nbf} {\textbf{n}}
\newcommand {\mbf} {\textbf{m}}

\newcommand{\eps}{\varepsilon}
\newcommand{\cA}{\mathcal{A}}
\newcommand{\cB}{\mathcal{B}}
\newcommand{\cX}{\mathcal{X}}
\newcommand{\R}{\mathbb{R}}
\newcommand{\N}{\mathbb{N}}

\newcommand{\bw}{\textbf{w}}
\newcommand{\normdot}{\| \cdot \|}
\def\w{\boldsymbol{\mathrm{w}}}
\DeclareMathOperator{\card}{Card}
\DeclareMathOperator{\sgn}{sgn}

\title{A general approximation lower bound in $L^p$ norm, with applications to feed-forward neural networks}

\author[1]{\bf El Mehdi Achour}
\author[1]{\bf Armand Foucault}
\author[2,1]{\bf Sébastien Gerchinovitz}
\author[1]{\bf François Malgouyres}
\affil[1]{Institut de Mathématiques de Toulouse ;  UMR5219 \\ Université de Toulouse ; CNRS  \\ UPS IMT F-31062 Toulouse Cedex 9, France}
\affil[2]{IRT Saint Exupéry, 3 rue Tarfaya, 31400 Toulouse, France}
\affil[ ]{}
\affil[ ]{\texttt{\{El\_mehdi.achour,armand.foucault,francois.malgouyres\} AT math.univ-toulouse.fr}}
\affil[ ]{\texttt{sebastien.gerchinovitz AT irt-saintexupery.com}}

\begin{document}

\maketitle

\begin{abstract}
We study the fundamental limits to the expressive power of neural networks. Given two sets $F$, $G$ of real-valued functions, we first prove a general lower bound on how well functions in $F$ can be approximated in $L^p(\mu)$ norm by functions in $G$, for any $p \geq 1$ and any probability measure $\mu$. The lower bound depends on the packing number of $F$, the range of $F$, and the fat-shattering dimension of $G$. We then instantiate this bound to the case where $G$ corresponds to a piecewise-polynomial feed-forward neural network, and describe in details the application to two sets $F$: Hölder balls and multivariate monotonic functions. Beside matching (known or new) upper bounds up to log factors, our lower bounds shed some light on the similarities or differences between approximation in $L^p$ norm or in sup norm, solving an open question by DeVore et al.~\cite{devore_hanin_petrova_2021}. Our proof strategy differs from the sup norm case and uses a key probability result of Mendelson \cite{971753}.
\end{abstract}

\section{Introduction}
\label{s:intro}

Neural networks are known for their great expressive power: in classification, they can interpolate arbitrary labels \cite{zhang2021understanding}, while in regression they have universal approximation properties \cite{Cyb89-UniversalApproximation,Hor91-ApproximationCapabilitiesMultilayer,LLPS93-NonpolynomialActivationFunction,kidger20-UniversalApproximationDeepNN}, with approximation rates that can outperform those of linear approximation methods \cite{yarotsky2018optimal,devore_hanin_petrova_2021}. Though the approximation problem is often only one part of the underlying learning problem (where generalization and optimization properties are also at stake), understanding the fundamental limits to the approximation properties of neural networks is key, both conceptually and for practical issues such as designing the right network architecture for the right problem.

\paragraph{Setting and related works.}
One way to quantify the expressive power of neural networks is through the following problem (some informal statements will be made more precise in the next sections). Let $G$ be the set of all functions $g_\bw: \cX \subset \R^d \to \R$ that can be represented by tuning the weights $\bw \in \R^W$ of a feed-forward neural network with a fixed architecture, and let $F$ be any set of real-valued functions on $\cX$. A natural question is: how well can functions $f \in F$ be approximated by functions $g_{\bw} \in G$? More precisely, given a norm $\normdot$ on functions, what is the order of magnitude of the (worst-case) \emph{approximation error of $F$ by $G$} defined by
\begin{equation}
    \label{eq:defapproxerror}
    \sup_{f \in F} \inf_{g_{\bw} \in G} \|f - g_{\bw}\| \;,
\end{equation}
and how small can it be given the numbers $W$, $L$ of weights and layers, and some properties of $F$?

Lower bounds on the approximation error \eqref{eq:defapproxerror} can be useful in several ways. They provide a limit to the best approximation accuracy that one can hope to achieve if the number of weights or layers of the network is constrained, and help design optimal architectures under these constraints. They also imply a lower bound on the minimal number of weights or layers to include in a network in order to approximate any function in $F$ with a given accuracy $\varepsilon$.

The case when $\normdot$ is the sup norm (defined by $\|f\|_{\infty} = \sup_{x \in \cX} |f(x)|$) is rather well understood at least in some special cases. For example, when $F$ is a Hölder ball of smoothness $s > 0$ (a.k.a. Hölder exponent) and the network uses the ReLU activation function, Yarotsky \cite{yarotsky2017error} derived a lower bound on \eqref{eq:defapproxerror} of the order of $W^{-2s/d}$, later refined to $(LW)^{-s/d}$ (up to log factors) by \cite{yarotsky2018optimal,yarotsky2020phase} when the depth of the network varies from $L=1$ to $L \approx W$. Using the bit extraction technique, these authors showed that these lower bounds are achievable (up to log factors) with a carefully designed ReLU network architecture. Refined results in terms of width and depth were obtained by \cite{shen21-optimalApproximationReLU} when $s \leq 1$, while some other activation functions were also studied in \cite{yarotsky2020phase}.

In this paper, we study \eqref{eq:defapproxerror} with the $L^p(\mu)$ norm, defined by $\|f\|_{L^p(\mu)} = \bigl(\int_X |f(x)|^p d \mu(x)\bigr)^{1/p}$, for $1 \leq p < +\infty$ and some probability measure $\mu$ on $\cX$. 
There is a qualitative difference between measuring the error in sup norm or in $L^p(\mu)$ norm, $p < +\infty$. In the former case, the error is small only if the approximation is good over the whole domain. In the latter case, the error can be small even if the approximation is inaccurate over a small portion of the domain.
Since the $L^p(\mu)$ approximation problem corresponds to approximating functions in $F$ in a more \say{average} sense than in sup norm, a natural question is whether the same accuracy can be achieved with a smaller network or not. Unfortunately, however, the proof strategies behind the lower bounds of \cite{yarotsky2017error,yarotsky2018optimal,yarotsky2020phase,shen21-optimalApproximationReLU} are specific to the sup norm (see Remark~\ref{rmk:proofsupnorm} in Section~\ref{section Hölder} for details).
DeVore et al. \cite{devore_hanin_petrova_2021} indeed commented: \say{When we move to the case $p < \infty$, the situation is even less clear [...] we cannot use the VC dimension theory for $L^p(\Omega)$ approximation. [...] What is missing vis-à-vis Problem $8.13$ is what the best bounds are and how we prove lower bounds for approximation rates in $L^p(\Omega)$, $p \neq \infty$.}

\textbf{Existing lower bounds in $L^p(\mu)$ norm.}
Several papers provided lower bounds in some special cases, under some restrictions on the set to approximate $F$, the neural network, the approximation metric, or the encoding map $f \in F \mapsto \bw(f) \in \R^W$.

When $F$ is a space of smoothness $s$,
a first result which is based on  \cite{devore1989optimal} states that when imposing the weights to depend continuously on the function to be approximated, one can not achieve a better approximation rate than $W^{-\frac{s}{d}}$.

For the same $F$, another result for $p=2$ and for activation functions which are continuous (\cite{MAIOROV199968,Maiorov1999OnTA}) proves a lower bound on the approximation of functions of smoothness $s$ on a compact of $\RR^d$, by one hidden-layer neural networks, of order $W^{-\frac{s}{d-1}}$.
A matching upper bound is proven for a particular activation function, which is sigmoidal but pathological (\cite{MAIOROV199981}).
For this same activation function, they prove that contrary to the one-hidden-layer case, there is no lower bound in the case of two-hidden-layer networks. The result is based on the Kolmogorov-Arnold superposition theorem.

In \cite{Siegel2021SharpBO}, the authors study approximation by shallow neural networks with bounded weights and activations of the form ReLU$^k$ for an integer $k$. They approximate the closure of the convex hull of shallow ReLU$^k$-neural networks with constrained weights. They obtain optimal lower bounds of order $W^{-\frac{1}{2} - \frac{2k + 1}{2d}}$ in any norm $\|\cdot\|_X$, where $X$ is a Banach space to which the approximation functions belong and such that these functions are uniformly bounded w.r.t. $\|\cdot\|_X$. Although we only consider approximation in $L^p(\mu)$ norm, our results complement the latter by addressing neural networks with unbounded weights and arbitrary depth, and general sets $F$.

Approximation lower bounds in $L^p(\mu)$ norm, $p \geq 1$, have also been studied in the quantized neural networks setting (networks with weights encoded with a fixed number of bits). In \cite{PETERSEN2018296}, under weak assumptions on the activation function, the authors prove a lower bound on the minimal number of nonzero weights $W$ that are required for a network to approximate a class of binary classifiers with $L^p$ error at most $\varepsilon$.
They show that $W$ is at least of the order $\varepsilon^{-\frac{p(d-1)}{\beta}} \log^{-1}_{2}\left(1/\varepsilon\right)$, where $\beta$ is a smoothness parameter. Later works including \cite{https://doi.org/10.48550/arxiv.1904.04789, guhring2021approximation} derive lower bounds for approximation by quantized networks in various norms.

\paragraph{Main contributions and outline of the paper.}

We prove lower bounds on the approximation error \eqref{eq:defapproxerror} in any $L^p(\mu)$ norm, for non-quantized networks of arbitrary depth, and general sets $F$. Our main contributions are the following.

In Section~\ref{sec: general approximation lower bound}, we first prove a general lower bound for any two sets $F$, $G$ of real-valued functions on a set~$\cX$ (Theorem~\ref{general_thm}). The lower bound depends on the packing number of $F$, the range of $F$, and the fat-shattering dimension of $G$.  We then derive a versatile corollary when $G$ corresponds to a piecewise-polynomial feed-forward neural network (Corollary~\ref{main_result}), solving the question by DeVore et al.~\cite{devore_hanin_petrova_2021}. Importantly, our proof strategy still relies on VC dimension theory, but differs from the sup norm case in using a key probability result of Mendelson \cite{971753}, to relate approximation in $L^p(\mu)$ norm with the fat-shattering dimension of $G$.

In Sections~\ref{section Hölder}--\ref{section monotonic} we apply this corollary to the approximation of two sets: Hölder balls and multivariate monotonic functions. Beside matching (known or new) upper bounds up to log factors, our lower bounds shed some light on the similarities or differences between approximation in $L^p$ norm or in sup norm. In particular, with ReLU networks, Hölder balls are not easier to approximate in $L^p$ norm than in sup norm. On the contrary, the approximation rate for multivariate monotonic functions depends on $p$.
In Section~\ref{sec:Barron}, we outline several other examples of function sets $F$ and $G$ for which the general lower bound (Theorem~\ref{general_thm}) can also be easily applied. Finally, some proofs are postponed to the supplement, while some details on other existing lower bound proof strategies are provided in the supplement, in Appendix~\ref{sec:relatedworks}.

\paragraph{Additional bibliographical remarks} There are many other related results that we did not mention to keep the focus on our specific approximation problem. For instance, depth separation results show that deep neural networks can approximate functions that cannot be as easily approximated by shallower networks (e.g., \cite{telgarsky2016benefits,vardi2021size}). Let us also mention the general results of \cite{YaBa-99-MinimaxRates}, which characterize minimax rates of estimation based on metric entropy conditions. Understanding the precise connections between these statistical results and our general approximation lower bound is an interesting question for the future.

\paragraph{Definitions and notation.}
We provide below some definitions and notation that will be used throughout the paper. We denote the set of positive integers $\{1,2,\ldots\}$ by $\N^*$ and let $\N := \N^* \cup\{0\}$. All sets considered in this paper will be assumed to be nonempty. 
We will not explicitly mention $\sigma$-algebras; for instance, by ``Let $\cX$ be a measurable space'' we mean that $\cX$ is a set implicitly endowed with a $\sigma$-algebra.

Let $p \in [1,+\infty]$ and $\cX$ be any measurable space endowed with a probability measure~$\mu$. For any measurable function $f:\cX \to \R$, the $L^p(\mu)$ norm of $f$ is defined by $\|f\|_{L^p(\mu)} = \bigl(\int_\cX |f(x)|^p d \mu(x)\bigr)^{1/p}$ (possibly infinite) if $p < +\infty$, and $\|f\|_{L^\infty(\mu)} = \esssup_{x \in \mathcal{X}}|f(x)|$. We will write $\lambda$ for the Lebesgue measure on $[0,1]^d$.

For any $\varepsilon > 0$, two functions $f_1, f_2$ are said to be \emph{$\varepsilon$-distant} in $\|\cdot\|$ if $\|f_1 - f_2\| > \varepsilon$.
Let $F$ be a set of functions from $\cX$ to $\RR$. A set $\{f_1, \ldots, f_N\} \subset F$ is said to be an $\varepsilon$-packing of $F$ in $\|\cdot\|$ (or just an $\varepsilon$-packing for short) if for any $i \neq j \in \{1, \ldots, N\}$, $f_i$ and $f_j$ are $\varepsilon$-distant in $\|\cdot\|$. The $\varepsilon$-packing number $M(\varepsilon,F,\|\cdot\|)$ is the largest cardinality of $\varepsilon$-packings (possibly infinite).

For $\gamma > 0$, we say that a set $S = \{x_1\, \ldots, x_N\} \subset \mathcal{X}$ is $\gamma$-\textit{shattered} by $F$ if there exists $r : S \rightarrow \mathbb{R}$ such that for any $E \subset S$, there exists $f \in F$ satisfying for all $i = 1, \ldots, N$, $f(x_i) \geq r(x_i) + \gamma$ if $x_i \in E$, and $f(x_i) \leq r(x_i) - \gamma$ if $x_i \notin E$. The $\gamma$-\textit{fat-shattering dimension of $F$}, denoted by $\fat_\gamma(F)$, is the largest number $N \geq 1$ for which there exists $S \subset \mathcal{X}$ of cardinality $N$ that is $\gamma$-shattered by $F$ (by convention, $\fat_\gamma(F)=0$ if no such set $S$ exists, while $\fat_\gamma(F)=+\infty$ if there exist sets $S$ of unbounded cardinality $N$). Similarly, we say that $S$ is \textit{pseudo-shattered} by $F$ if there exists $r : S \rightarrow \mathbb{R}$ such that for any $E \subset S$, there exists $f \in F$ satisfying for all $i=1, \ldots, N$, $f(x_i) \geq r(x_i)$ if $x_i \in E$, and $f(x_i) < r_i$ if $x_i \notin E$. The \textit{pseudo-dimension} $\Pdim(F)$ is the largest number $N \geq 1$ for which there exists $S \subset \mathcal{X}$ of cardinality $N$ that is pseudo-shattered by $F$ (same conventions).\footnote{By definition, note that $\gamma \mapsto \fat_\gamma(F)$ is non-increasing and that $\fat_\gamma(F) \leq \Pdim(F)$ for all $\gamma>0$.}

A formal definition of feed-forward neural networks is recalled in Appendix~\ref{sec:deffeed-forward}. In short, in this paper, a \emph{feed-forward neural network architecture} $\mathcal{A}$ of depth $L \geq 1$ is a directed acyclic graph with $d \geq 1$ input neurons, $L-1$ hidden layers (if $L \geq 2$), and an output layer with only one neuron. Skip connections are allowed, i.e., there can be connections between non-consecutive layers. Given an activation function $\sigma:\R \to \R$, a feed-forward neural network architecture $\mathcal{A}$, and a vector $\bw \in \R^W$ of weights assigned to all edges and non-input neurons (linear coefficients and biases), the network computes a function $g_{\bw}: \R^d \to \R$ defined by recursively computing affine transformations for each hidden or output neuron, and then applying the activation function $\sigma$ for hidden neurons only (see Appendix~\ref{sec:deffeed-forward} for more details). Finally, we define $H_{\mathcal{A}} := \{g_{\bw}: \bw \in \R^W\}$ to be the set of all functions that can be represented by tuning all the weights assigned to the network.

A function $\sigma:\R \to \R$ is \emph{piecewise-polynomial} on $K \geq 2$ pieces, with maximal degree $\nu \in \N$, if there exists a partition $I_1, \ldots,I_{K}$ of $\R$ into $K$ nonempty intervals, such that $\sigma$ restricted on each $I_j$ is polynomial with degree at most $\nu$ (in particular, $\sigma$ can be discontinuous).

\section{A general approximation lower bound in $L^p(\mu)$ norm}\label{sec: general approximation lower bound}

In this section, we provide our two main results: a general lower bound on the $L^p(\mu)$ approximation error of $F$ by $G$, i.e., $\sup_{f \in F}\inf_{g \in G}\|f - g\|_{L^p(\mu)}$, and a corollary when $G$ corresponds to a feed-forward neural network with a piecewise-polynomial activation function. The weak assumptions on $F$ make the last result applicable to a wide range of cases of interest, as shown in Sections~\ref{section Hölder}--\ref{sec:Barron}.

\subsection{Main results}

Our generic lower bound reads as follows, and is proved in Section \ref{main_result_proof}. We follow the conventions $0 \times \log^2(0) = 0$ and $P^{-\frac{1}{\alpha}}\log^{-\frac{2}{\alpha}}(P) = +\infty$ when $P=1$. \\[0.1cm]

\begin{theorem}
\label{general_thm}
Let $1 \leq p < +\infty$ and $\cX$ be a measurable space endowed with a probability measure~$\mu$. Let $F$, $G$ be two sets of measurable functions from $\cX$ to $\R$, such that all functions in $F$ have the same range $[a,b]$ for some $a<b$, and such that $\fat_\gamma(G) < +\infty$ for all $\gamma>0$. Then, there exists a constant $c>0$ depending only on $p$ such that
\begin{align}
    \sup_{f \in F} \inf_{g \in G} \|f - g\|_{L^p(\mu)}
    \geq  \inf\left\{\varepsilon>0: \log M\!\left(3\varepsilon, F, \|\cdot\|_{L^p(\mu)}\right) \leq c \fat_{{\frac{\varepsilon}{32}}}(G) \log^2\!\left(\frac{2 \fat_{{\frac{\varepsilon}{32}}}(G)}{\varepsilon/(b-a)}\right)\right\}. \label{lower bound thm general}
\end{align}
In particular, if $\log M\bigl(\eps, F, \|\cdot\|_{L^p(\mu)}\bigr) \geq  c_0 \eps^{-\alpha}$ for some $c_0,\eps_0,\alpha>0$ and all $\eps \leq \eps_0$, and if $\Pdim(G) < +\infty$, then there exist constants $c_1,\eps_1>0$ depending only on $b-a$, $p$, $c_0$, $\varepsilon_0$ and $\alpha$ such that
\begin{align}
    \label{main_result_equation}
    \sup_{f \in F} \inf_{g \in G} \|f - g\|_{L^p(\mu)} \geq \min\left\{\varepsilon_1, \ c_1 \Pdim(G)^{-\frac{1}{\alpha}}\log^{-\frac{2}{\alpha}}\bigl(\Pdim(G)\bigr)\right\} \;.
\end{align}
\end{theorem}

The first lower bound \eqref{lower bound thm general} is generic but requires solving an inequation.\footnote{\label{ft:solutioninequation}Note that any $\eps \geq (b-a)/3$ is a solution to this inequation, since $\log M\bigl(3\eps, F, \|\cdot\|_{L^p(\mu)}\bigr) = \log(1)=0$ (because all functions in $F$ are $[a,b]$-valued) and $c \fat_{{\frac{\varepsilon}{32}}}(G) \geq 0$. Therefore, the right-hand side of \eqref{lower bound thm general} is at most $(b-a)/3$.} In \eqref{main_result_equation} we solve this inequation when $\log M\bigl(\eps, F, \|\cdot\|_{L^p(\mu)}\bigr)$ grows at least polynomially in $1/\eps$ (which is typical of nonparametric sets) and when $G$ has finite pseudo-dimension $\Pdim(G)$. Though we will restrict our attention to such cases in all subsequent sections, we stress that the first bound should have broader applications. A first example is when $\Pdim(G)=+\infty$ but $\fat_{\gamma}(G)<+\infty$ for all $\gamma>0$ (e.g., for RKHS \cite{belkin18a-RKHS}). The first bound should also be useful to prove (slightly) tighter lower bounds when $\log M\bigl(\eps, F, \|\cdot\|_{L^p(\mu)}\bigr)$ has a (slightly) different dependency on $1/\eps$ (e.g., of the order of $\varepsilon^{-\alpha}\log^\beta\left(1/\varepsilon\right)$ as when $F$ is the set of all multivariate cumulative distribution functions \cite{blei07-metricEntropyHighDim}).

In the rest of the paper, we focus on the important special case when the approximation set $G$ is the set $H_{\mathcal{A}}$ of all real-valued functions that can be represented by tuning the weights of a feed-forward neural network with fixed architecture $\mathcal{A}$ and a piecewise-polynomial activation function.
By combining Theorem~\ref{general_thm} with known bounds on the pseudo-dimension \cite{JMLR:v20:17-612}, we obtain the following corollary, which bounds the approximation error in terms of the number $W$ of weights and the depth~$L$ (i.e., the number of hidden and output layers). The proof is postponed to Appendix~\ref{sec:proofcorollary-main}.\\[0.2cm]

\begin{coro}
\label{main_result}
Let $1 \leq p < +\infty$, $d \geq 1$ and $\cX$ be a measurable subset of $\R^d$ endowed with a probability measure~$\mu$. Let $F$ be a set of measurable functions from $\mathcal{X}$ to $[a,b]$ (for some real numbers $a<b$), such that ${\log M\bigl(\eps, F, \|\cdot\|_{L^p(\mu)}\bigr) \geq  c_0 \eps^{-\alpha}}$ for some $c_0,\eps_0,\alpha>0$ and all $\eps \leq \eps_0$.

Let $\sigma:\R \to \R$ be any piecewise-polynomial activation function of maximal degree $\nu \in \mathbb{N}$ on $K \geq 2$ pieces. Then, there exist $W_{\min} \in \mathbb{N}^*$ and $c_1,c_2,c_3>0$ such that, for any $W \geq W_{\min}$, any $L \geq 1$, and any fixed feed-forward neural network architecture $\mathcal{A}$ of depth $L$ with $W$ weights, the set $H_{\mathcal{A}}$ of all real-valued functions on $\mathcal{X}$ that can be represented by the network (cf. Section~\ref{s:intro}) satisfies
\begin{align}
\label{eq:lowerbound-feed-forward}
    \sup_{f \in F} \inf_{g \in H_{\mathcal{A}}} \|f - g\|_{L^p(\mu)} \geq
    \left\{
    \begin{array}{l l}
         c_1  W^{-\frac{2}{\alpha}} \log^{-\frac{2}{\alpha}}(W) & \text{ if } \nu \geq 2 \,, \\
         c_2  (LW)^{-\frac{1}{\alpha}} \log^{-\frac{3}{\alpha}}(W)& \text{ if } \nu=1 \,, \\
         c_3  W^{-\frac{1}{\alpha}} \log^{-\frac{3}{\alpha}}(W)& \text{ if } \nu=0 \,.
    \end{array}
    \right.
\end{align}
\end{coro}

There are equivalent ways to write the above corollary. For example, given a target accuracy $\eps>0$ and a depth $L\geq 1$, \eqref{eq:lowerbound-feed-forward} yields a lower bound on the minimum number $W$ of weights that are needed to get $\sup_{f \in F} \inf_{g \in H_{\mathcal{A}}} \|f - g\|_{L^p(\mu)} \leq \eps$. Some earlier approximation results were written this way (e.g., \cite{yarotsky2017error,PETERSEN2018296}).

\subsection{Proof of Theorem \ref{general_thm}}
\label{main_result_proof}

In order to prove Theorem \ref{general_thm}, we need two inequalities.
The first one is straightforward (and appeared within proofs, e.g., in \cite{yarotsky2020phase}), but formalizes the key idea that if $G$ approximates $F$ with error $\varepsilon$, then $G$ has to be at least as large as $F$.
We use the conventions $\log(+\infty) = +\infty$ and $+\infty \leq +\infty$.

\begin{lemma}
\label{lemma_entropy_G_vs_F}
Let $p \geq 1$ and $\cX$ be a measurable space endowed with a probability measure $\mu$. Let $F$, $G$ be two sets of measurable functions from $\cX$ to $\R$.
If $\sup_{f \in F} \inf_{g \in G} \|f - g\|_{L^p(\mu)} < \varepsilon$, then
$$
\log M\!\left(3\varepsilon, F, \|\cdot\|_{L^p(\mu)}\right) \leq \log M\!\left(\varepsilon, G, \|\cdot\|_{L^p(\mu)}\right)\;.
$$

\end{lemma}

\begin{proof}
Let $P_F = \{f_1, \ldots, f_N\}$ be a $3\varepsilon$-packing of $F$, with $N \geq 1$.
Let $P_G = \{g_1, \ldots, g_N\}$ be a subset of $G$ such that $\|f_i - g_i\|_{L^p(\mu)} \leq \varepsilon$ for all $i$. Note that the existence of such a $P_G$ is guaranteed by the assumption $\sup_{f \in F} \inf_{g \in G} \|f - g\|_{L^p(\mu)} < \varepsilon$. Since the $f_i$'s are pairwise $3\varepsilon$-distant in $L^p(\mu)$, the triangle inequality entails that the $g_i$'s are also at least pairwise $\varepsilon$-distant in $L^p(\mu)$. Therefore, $P_G$ is an $\varepsilon$-packing of $G$, and the result follows.
\end{proof}

The next inequality is a fundamental probability result due to Mendelson \cite{971753}. It bounds from above the $\varepsilon$-packing number in $L^p(\mu)$ norm of any uniformly bounded function set in terms of its fat-shattering dimension. Crucially, the inequality holds for finite $p \geq 1$, as opposed to the lower bound strategy of Yarotsky \cite{yarotsky2017error,yarotsky2018optimal} (see also \cite{devore_hanin_petrova_2021}), that relates the VC-dimension with the approximation error in sup norm.
The next statement is a slight generalization of a result of \cite{971753} initially stated for $[a,b]=[0,1]$ and for Glivenko-Cantelli classes $G$ (see Appendix~\ref{sec:rescaling} for details).

\begin{prop}[\cite{971753}, Corollary~3.12]
\label{result_mendelson}
Let $G$ be a set of measurable functions from a measurable space $\mathcal{X}$ to $[a,b]$ (for some real numbers $a<b$), and such that $\fat_{\gamma}(G) < +\infty$ for all $\gamma > 0$. Then for any $1 \leq p < +\infty$, there exists $c > 0$ depending only on $p$ such that for every probability measure $\mu$ on $\cX$ and every $\varepsilon > 0$,
\begin{align}
\label{equation_mendelson}
\log M\!\left(\varepsilon, G, \|\cdot\|_{L^p(\mu)}\right) \leq c \fat_{\frac{\varepsilon}{32}}(G) \log^2 \!\left(\frac{2 (b-a) \fat_{\frac{\varepsilon}{32}}(G)}{\varepsilon}\right) \;.
\end{align}
\end{prop}

Refinements of this inequality were proved in specific cases such as the $L^2(\mu)$ norm \cite{mendelson03-entropyCombinatorialDimension} (see also \cite{guermeur17-LpnormSauer} for empirical $L^p(\mu_n)$ norms). However, using the result of \cite{mendelson03-entropyCombinatorialDimension} when $p=2$ would only yield a minor logarithmic improvement in the lower bound of Theorem~\ref{general_thm}.

\begin{proof}[Proof (of Theorem \ref{general_thm})]
\label{main_result_proof_bis}
\textbf{Part 1.} We start by proving \eqref{lower bound thm general}, using Proposition \ref{result_mendelson} as a key argument. Since functions in $G$ are not necessarily uniformly bounded, we will apply Proposition \ref{result_mendelson} to the \say{clipped version of $G$}. More precisely, for any function $g \in G$, we define its clipping (truncature) to $[a,b]$ as the function $\tilde{g}:\cX \to \R$ given by $\tilde{g}(x) = \min(\max(a,g(x)), b)$ for all $x \in \mathcal{X}$. We then set $G_{[a,b]} = \left\{ \tilde{g}: g \in G \right\}$, which by construction consists of functions that are all $[a,b]$-valued.

Noting that clipping can only help since elements of $F$ are $[a,b]$-valued (see Lemma~\ref{lemma_supinf_G} in the supplement, Appendix~\ref{sec:clipping}), we have
\begin{equation}
    \label{eq:clipping}
    \sup_{f \in F}\inf_{g \in G} \|f - g\|_{L^p(\mu)} \geq \sup_{f \in F} \inf_{\tilde{g} \in G_{[a,b]}} \|f - \tilde{g}\|_{L^p(\mu)} \;.
\end{equation}

Setting $\Delta := \sup_{f \in F} \inf_{\tilde{g} \in G_{[a,b]}} \|f - \tilde{g}\|_{L^p(\mu)}$, we now show that $\Delta$ is bounded from below by the right-hand side of \eqref{lower bound thm general}. To that end, it suffices to show that every $\eps > \Delta$ is a solution to the inequation
\begin{equation}
\label{eq:inequation}
    \log M\!\left(3\varepsilon, F, \|\cdot\|_{L^p(\mu)}\right) \leq c \fat_{\frac{\varepsilon}{32}}(G) \log^2 \left(\frac{2 (b-a) \fat_{\frac{\varepsilon}{32}}(G)}{\varepsilon}\right) \;.
\end{equation}
The last inequality is true whenever $\eps \geq (b-a)/3$ (see Footnote~\ref{ft:solutioninequation}). We only need to prove \eqref{eq:inequation} when $\Delta < \eps < (b-a)/3$. In this case, by definition of $\Delta$ and by Lemma~\ref{lemma_entropy_G_vs_F} applied to $G_{[a,b]}$, we have
\begin{align}
    \log M\!\left(3\varepsilon, F, \|\cdot\|_{L^p(\mu)}\right) & \leq \log M\!\left(\varepsilon, G_{[a,b]}, \|\cdot\|_{L^p(\mu)}\right) \nonumber\\
    & \leq c \fat_{\frac{\varepsilon}{32}}(G_{[a,b]}) \log^2 \left(\frac{2 (b-a) \fat_{\frac{\varepsilon}{32}}(G_{[a,b]})}{\varepsilon}\right) \nonumber\\
    & \leq c \fat_{\frac{\varepsilon}{32}}(G) \log^2 \left(\frac{2 (b-a) \fat_{\frac{\varepsilon}{32}}(G)}{\varepsilon}\right) \;,\label{eq:thirdineq}
\end{align}
where the second inequality follows from Proposition \ref{result_mendelson} (note from Lemma~\ref{lemma_fatG} in the supplement, Appendix~\ref{sec:clipping} that $\fat_{\gamma}(G_{[a,b]}) \leq \fat_{\gamma}(G)$ for all $\gamma>0$, which is finite by assumption), and where \eqref{eq:thirdineq} follows from the next remark. Either $\fat_{\frac{\eps}{32}}(G_{[a,b]}) = 0$, and \eqref{eq:thirdineq} is true by the convention $0 \times \log^2(0) = 0$ and $c\fat_{\frac{\eps}{32}}(G)\geq 0$. Either $\fat_{\frac{\eps}{32}}(G_{[a,b]}) \geq 1$, and \eqref{eq:thirdineq} follows from  $t \mapsto c t \log^2\bigl( \frac{2(b-a)t}{\varepsilon}\bigr)$ being non-decreasing on $[\varepsilon/(2(b-a)),+\infty)$ and $\varepsilon/(2(b-a)) \leq 1/6 \leq 1 \leq \fat_{\frac{\eps}{32}}(G_{[a,b]}) \leq \fat_{\frac{\eps}{32}}(G)$. To conclude, every $\eps > \Delta$ satisfies \eqref{eq:inequation}, which implies that $\Delta$ is bounded from below by the right-hand side of \eqref{lower bound thm general}. Combining with \eqref{eq:clipping} concludes the proof of \eqref{lower bound thm general}.

\textbf{Part 2.} Set $\eps_1'=\min\bigl\{\frac{\eps_0}{3}, 2(b-a)\bigr\}$. We now derive \eqref{main_result_equation} from \eqref{lower bound thm general}. To that end, setting $P=\Pdim(G)$, we show that every $\eps>0$ satisfying \eqref{eq:inequation} is such that $\eps \geq \min\bigl\{\eps_1, \ c_1  P^{-\frac{1}{\alpha}}\log^{-\frac{2}{\alpha}}(P)\bigr\}$, where $\eps_1 \in (0,\eps_1']$ and $c_1>0$ will be defined later. Since the claimed lower bound on $\eps$ is true when $\eps \geq \eps_1'$, in the sequel we consider any solution $\eps$ to \eqref{eq:inequation} such that $0 < \eps < \eps_1'$ (if such a solution exists).

By the assumption on $\log M\bigl(u, F, \|\cdot\|_{L^p(\mu)}\bigr)$ for $u=3\eps \leq \eps_0$, and then using \eqref{eq:inequation}, we have, setting $r = 2(b-a)$,
\begin{equation*}
c_0 (3\varepsilon)^{-\alpha}  \leq \log M\!\left(3\varepsilon, F, \|\cdot\|_{L^p(\mu)}\right) \leq c \fat_{\frac{\varepsilon}{32}}(G) \log^2 \left(\frac{r \fat_{\frac{\varepsilon}{32}}(G)}{\varepsilon}\right) \leq c P \log^2 \left(\frac{r P}{\varepsilon}\right) \;,
\end{equation*}
where the last inequality is because $t \mapsto c t \log^2\bigl( \frac{rt}{\varepsilon}\bigr)$ is non-decreasing on $[\eps/r,+\infty)$, with $\eps/r \leq 1$, and $1 \leq \fat_{\frac{\varepsilon}{32}}(G) \leq \Pdim(G)=P$ (the lower bound of $1$ follows from $c_0 (3\varepsilon)^{-\alpha}>0$).

Solving the inequation $c_0 (3\varepsilon)^{-\alpha} \leq c P \log^2 (r P/\eps)$ for $\eps$ (see Appendix \ref{function_study} for details), we get
\begin{equation}
    \label{main_result_proof_conclusion}
    \varepsilon \geq \min\bigl\{\varepsilon_1'', \ c_1 P^{-\frac{1}{\alpha}} \log^{-\frac{2}{\alpha}}P\bigr\} \;,
\end{equation}
for some constants $\eps_1'', c_1 > 0$ depending only on $p$, $c_0$, $b-a$ and $\alpha$. Setting $\varepsilon_1 = \min\{\eps_1'', \eps_1'\}$ and noting that $\varepsilon'_1$ only depends on  $\varepsilon_0$ and $b-a$, we conclude the proof. 
\end{proof}


\section{Approximation of Hölder balls by feed-forward neural networks}
\label{section Hölder}

In this section, we apply Corollary \ref{main_result} to establish nearly-tight lower bounds for the approximation of unit Hölder balls by feed-forward neural networks. Our main result is Proposition \ref{borne inf Hölder}, which solves an open question by \cite{devore_hanin_petrova_2021}.

Throughout the section, for any $s > 0$, we denote by $n$ and $\alpha$ the unique members of the decomposition $s = n + \alpha$ such that $n \in \mathbb{N}$ and $0 < \alpha \leq 1$. 
 
For a set $\mathcal{X} \subset \mathbb{R}^d$, we follow \cite{yarotsky2020phase} and define the Hölder space $\mathcal{C}^{n,\alpha}(\mathcal{X})$ as the space of $n$ times continuously differentiable functions with finite norm
$$
\|f\|_{\mathcal{C}^{n,\alpha}} = \max \left\{\max_{\textbf{n}:|\textbf{n}| \leq n}\|D^{\textbf{n}}f\|_{\infty}, \max_{\textbf{n}: |  \textbf{n}|=n} \sup_{x \neq y} \frac{\left|D^{\textbf{n}}f(x) - D^{\textbf{n}}f(y)\right|}{\|x - y\|^{\alpha}_2}
\right\},
$$
where, for $\nbf = (n_1,\cdots, n_d)\in\NN^d$, $D^{\textbf{n}}f = \left(\frac{\partial }{\partial x_1}\right)^{n_1}\cdots \left(\frac{\partial }{\partial x_d}\right)^{n_d} f$ denotes the $|\textbf{n}|$-order partial derivative of $f$. We denote
\[F_{s,d} = \{f \in \mathcal{C}^{n,\alpha}([0,1]^d) : \|f\|_{\mathcal{C}^{n,\alpha}} \leq 1\}.
\]

Let $\lambda$ denote the Lebesgue measure over $[0,1]^d$. In this section, we provide nearly matching upper and lower bounds for the $L^p(\lambda)$ approximation error of elements of $F_{s,d}$ by feed-forward ReLU neural networks. The bounds are expressed in terms of the number of weights of the network.

\subsection{Known bounds on the sup norm approximation error}

\cite{yarotsky2020phase} gives matching (up to a certain constant) lower and upper bounds of the sup norm approximation error of the elements of $F_{s,d}$ by feed-forward ReLU neural networks.

\begin{prop}[\cite{yarotsky2020phase}]
Let $d \in \NN^*$,  $s>0$, $\gamma \in \left(\frac{s}{d},\frac{2s}{d}\right]$. Consider $n\in\NN$ and $\alpha\in (0,1]$ such that $s = n + \alpha$.

There exist positive constants $W_{\emph{min}}$ and $c_1$, depending only on $d$ and $n$, such that for any integer $W \geq W_{\emph{min}}$, there exists a feed-forward ReLU neural network architecture $\mathcal{A}$ with $L=O(W^{\gamma \frac{d}{s}-1})$ layers and $W$ weights such that
\begin{equation}
    \label{borne_sup_yarotsky_2020}
    \sup_{f \in F_{s,d}}\inf_{g \in H_{\mathcal{A}}} \|f - g\|_{\infty} \leq 
    c_1 W^{-\gamma}.
\end{equation}

In the meantime, there exists a constant $c_2>0$ depending only on $d$ and $n$ such that, for any feed-forward neural network  architecture $\mathcal{A}$ with $W$ weights and $L = o(W^{\gamma\frac{d}{s} - 1}/\log W)$ layers and for the ReLU activation function,
\begin{equation}
    \label{borne_inf_yarotsky_2020}
    \sup_{f \in F_{s,d}} \inf_{g \in H_{\mathcal{A}}} \|f - g\|_{\infty} \geq c_2 W^{-\gamma}.
\end{equation}
\end{prop}

It is worth stressing that, for any probability measure $\mu$ on $[0,1]^d$, the upper bound (\ref{borne_sup_yarotsky_2020}) is automatically generalized to any smaller $L^p(\mu)$ norm, when $1 \leq p < +\infty$. However, the lower bound (\ref{borne_inf_yarotsky_2020}) does not immediately apply when $\|\cdot\|_{\infty}$ is replaced with $\|\cdot\|_{L^p(\mu)}$, $1 \leq p < +\infty$. The lower bound of the next subsection shows that, in this setting, approximation in $L^p(\lambda)$ norm is not easier than in sup norm, solving an open question of DeVore et al.~\cite{devore_hanin_petrova_2021}.

\subsection{Nearly-matching lower bounds of the $L^p(\lambda)$ approximation error}

We first state a lower bound on the packing number of $F_{s,d}$, which is rather classical though hard to find in this specific form (see \cite{Birman_1967} for the $L^\infty$ norm, or \cite{edmunds_triebel_1996} for other Sobolev-type norms).  For the sake of completeness, we give a proof of Lemma \ref{lemme lower bound metric entropy Hölder} in the supplement, Appendix \ref{proof of lemme lower bound metric entropy Hölder}. 

\begin{lemma}\label{lemme lower bound metric entropy Hölder} Let $s>0$, $d\in\NN^*$ and $1 \leq p < +\infty$. There exist constants $\varepsilon_0, c_0>0$ such that for any $0 < \varepsilon \leq \varepsilon_0$,
\begin{equation}\label{lemme2_eq}
    \log M \left(\varepsilon,F_{s,d}, \| \cdot \|_{L^p(\lambda)}\right) \geq c_0 \varepsilon^{-\frac{d}{s}}.
\end{equation}

\end{lemma}

Given Lemma \ref{lemme lower bound metric entropy Hölder}, we can use Corollary \ref{main_result} to establish the next proposition and obtain the lower bound on the $L^p(\lambda)$ approximation error.

\begin{prop}
\label{borne inf Hölder}
Let $d \in \NN^*$,  $s>0$, $\gamma \in \left(\frac{s}{d},\frac{2s}{d}\right]$ and $1 \leq p < +\infty$.
Consider $n\in\NN$ and $\alpha\in (0,1]$ such that $s = n + \alpha$.

Let $\sigma : \mathbb{R} \rightarrow \mathbb{R}$ be a piecewise-affine function, and $c>0$. 
Then, there exist $c_1>0$ and $W_{\emph{min}} \in \mathbb{N}^*$
(depending only on $s$, $d$, $p$, $\sigma$ and $c$)
such that for any architecture $\mathcal{A}$ of depth $1 \leq L \leq cW^{\gamma \frac{d}{s}-1} $ with $W\geq W_{\emph{min}}$ weights, and for the activation $\sigma$, the set $H_{\mathcal{A}}$ (cf. Section~\ref{s:intro}) satisfies

\begin{equation}\label{notre_borne_holder}
    \sup_{f \in F_{s,d}} \inf_{g \in {H}_\mathcal{A}} \|f - g\|_{L^p(\lambda)} \geq
         c_1  {W^{-\gamma}} \log^{-\frac{3s}{d}}(W) \;.
\end{equation}
\end{prop}

Note that the rate of the lower bound does not depend on~$p$. Note also that, when the activation function is ReLU (which is piecewise-affine), we obtain a lower bound which matches the upper bound of the previous subsection up to logarithmic factors.

\begin{proof}
From Lemma \ref{lemme lower bound metric entropy Hölder}, there exist $\varepsilon_0, c_0>0$ such that $\log M\left(\varepsilon, F_{s,d}, \|\cdot\|_{L^p(\lambda)}\right) \geq c_0 \varepsilon^{-\frac{d}{s}}$ for all $0 < \varepsilon \leq \varepsilon_0$.
Combining with Corollary \ref{main_result} and using $L \leq c W^{\gamma \frac{d}{s}-1}$ concludes the proof.
\end{proof}

\begin{remark}[Comparison with existing proof strategies in sup norm.]
\label{rmk:proofsupnorm}
We would like to highlight a key difference between the proof of Proposition \ref{borne inf Hölder} and the lower bound proof strategies of \cite{yarotsky2017error,yarotsky2018optimal,yarotsky2020phase,shen21-optimalApproximationReLU} that are specific to the sup norm. Their overall argument is roughly the following: if $G$ can approximate any $f \in F$ in sup norm at accuracy $\eps>0$, since $F$ contains many \say{oscillating} functions with oscillation amplitude roughly $\eps$, then so must be the case for $G$ (the sup norm is key here: \textbf{all} oscillations of any $f \in F$ are well approximated). Therefore, a small $\eps$ implies a large $\VCdim(G)$, which by contrapositive enables to lower bound the approximation error \eqref{eq:defapproxerror} with a decreasing function of $\VCdim(G)$, and therefore as a function of $L$ and $W$.
In contrast, in the proof of Theorem~\ref{general_thm}, the key probability result of Mendelson (Proposition~\ref{result_mendelson}) enables us to show that, even if the oscillations of any $f \in F$ are only well approximated \textbf{on average} (in $L^p(\mu)$ norm) by $G$, then $\Pdim(G)$ must be large when $\eps$ is small. The conclusion is then the same: the approximation error in $L^p(\mu)$ norm can be lower bounded as a function of $\Pdim(G)$, and therefore in terms of $L$, $W$. This solves the question of DeVore et al.~\cite{devore_hanin_petrova_2021} mentioned in the introduction, showing in particular that VC dimension theory can (surprisingly) be useful to prove $L^p$ approximation lower bounds.
\end{remark}


\section{Approximation of monotonic functions by feed-forward neural networks}
\label{section monotonic}

In this section, we consider the problem of approximating the set $\mathcal{M}^d$ of all non-decreasing functions from $[0,1]^d$ to $[0,1]$. These are functions $f: [0,1]^d \to [0,1]$ that are non-decreasing along any line parallel to an axis, i.e., such that, for all $x, y \in [0,1]^d$,
$$
x_i \leq y_i, \ \forall i = 1, \ldots, d \implies f(x) \leq f(y) \;.
$$
Monotonic functions are an interesting case study for at least two reasons. First, they naturally appear in physics or engineering applications (consider for instance the braking distance of a vehicle as a function of variables such as the speed, the total load or the drag coefficient). Second, as will be shown in this section, because their sets of discontinuities can have ``complex'' shapes in dimension $d \geq 2$, monotonic functions provide a good example for which the approximation by feed-forward neural networks is hopeless in sup norm, but can be achieved in $L^p(\lambda)$ norm.

Next we focus on the approximation of $\mathcal{M}^d$ with Heaviside feed-forward neural networks. After proving an impossibility result for the sup norm, we show that the weaker goal of approximating $\mathcal{M}^d$ in $L^p(\lambda)$ norm is feasible, and derive nearly matching lower and upper bounds. Interestingly, the approximation rates depend on $p \geq 1$, which is in sharp contrast with the case of Hölder balls, that are not easier to approximate in $L^p(\lambda)$ norm than in sup norm (see Section \ref{section Hölder}). 

\subsection{Warmup: an impossibility result in sup norm}
\label{sec:monotonic-warmup}
We start this section by showing that approximating monotonic functions of $d \geq 2$ variables in sup norm is impossible with Heaviside neural networks.

\begin{prop}\label{borne inf norme infinie}
For any neural network architecture $\mathcal{A}$ with the  Heaviside activation, the set $H_\mathcal{A}$ (cf. Section~\ref{s:intro}) satisfies
$$
\sup_{f \in \mathcal{M}^d} \inf_{g \in H_{\mathcal{A}}} \|f - g\|_{\infty} \geq \frac{1}{2}.
$$
\end{prop}
The proof of Proposition \ref{borne inf norme infinie} is postponed to the supplement, Appendix \ref{preuve borne inf norme infinie}.
We show a slightly stronger result, by exhibiting a single function $f\in \mathcal{M}^d$ such that the lower bound of $\frac{1}{2}$ holds simultaneously for all network architectures.

Next we study the approximation of $\mathcal{M}^d$ in $L^p(\lambda)$ norm.

\subsection{Lower bound in $L^p(\lambda)$ norm}

We start by proving a lower bound, as a direct consequence of Corollary \ref{main_result} and a lower bound on the packing number due to \cite{GAO20071751}.

\begin{prop}
\label{borne inf monotone}
Let $1 \leq p < +\infty$, $d \geq 1$, and let $\alpha = \max\{d, (d-1)p\}$.
Let $\sigma : \mathbb{R} \rightarrow \mathbb{R}$ be a piecewise-polynomial function having maximal degree $\nu \in \mathbb{N}$. Then, there exist positive constants $c_1, c_2, c_3, W_{\emph{min}}$ (depending only on $d$, $p$, and $\sigma$) such that for any architecture $\mathcal{A}$ of depth $L \geq 1$ with $W\geq W_{\emph{min}}$ weights, and for the activation $\sigma$, the set $H_{\mathcal{A}}$ (cf. Section~\ref{s:intro}) satisfies
\begin{align}
\label{eq:lowerbound-monotone}
    \sup_{f \in \mathcal{M}^d} \inf_{g \in H_{\mathcal{A}}} \|f - g\|_{L^p(\lambda)} \geq
    \left\{
    \begin{array}{l l}
         c_1  W^{-\frac{2}{\alpha}} \log^{-\frac{2}{\alpha}}(W) & \text{ if } \nu \geq 2 \,, \\
         c_2  (LW)^{-\frac{1}{\alpha}} \log^{-\frac{3}{\alpha}}(W)& \text{ if } \nu=1 \,, \\
         c_3  W^{-\frac{1}{\alpha}} \log^{-\frac{3}{\alpha}}(W)& \text{ if } \nu=0 \,.
    \end{array}
    \right.
\end{align}
\end{prop}

Note that, contrary to the case of Hölder balls (Section \ref{section Hölder}), the rate of the lower bound depends on $p$ through $\alpha = \max\{d, (d-1)p\}$.

\begin{proof}
From \cite{GAO20071751}, there exist constants $\varepsilon_0,c_0>0$ such that for $\varepsilon \leq \varepsilon_0$, $\log M\left(\varepsilon, \mathcal{M}_d, \|\cdot\|_{L^p(\lambda)}\right) \geq c_0 \varepsilon^{-\alpha}$.
Using Corollary \ref{main_result}, we obtain the result.
\end{proof}

\subsection{Nearly-matching upper bound in $L^p(\lambda)$ norm}
\label{sec:monotonic-upperbound}

To the best of our knowledge, there does not exist any upper-bound of the $L^p(\lambda)$ approximation error of $\mathcal{M}^d$ with feed-forward neural networks. Checking that all the lower-bounds of Proposition \ref{borne inf monotone} are tight is out of the scope of this paper and we leave it for future research\footnote{Obtaining an upper-bound for ReLU networks seems challenging. For example, the bit extraction technique used in \cite{yarotsky2018optimal} to find a sharp upper bound heavily relies on the local smoothness assumption of the function to approximate, which is not satisfied in general for monotonic functions.}. However, we establish in the next proposition upper-bounds of the $L^p(\lambda)$ approximation error of $\mathcal{M}^d$ with feed-forward neural networks with the Heaviside activation function. This shows that, for the $L^p(\lambda)$ approximation error, the lower-bound obtained in \eqref{eq:lowerbound-monotone}, for $\nu=0$, is tight up to logarithmic factors. The next proposition follows by reinterpreting a metric entropy upper bound of \cite{GAO20071751} in terms of Heaviside neural networks. The proof is postponed to Appendix \ref{preuve borne sup monotone} in the supplement.

\begin{prop}\label{borne sup monotone}

Let $1 \leq p < +\infty$, $d \in \mathbb{N} \setminus \{0,1\}$ and let $\alpha = \max\{d, (d-1)p\}$. There exist positive constants $W_{\emph{min}}$ and $c$, depending only on $d$ and $p$, such that for any integer $W \geq W_{\emph{min}}$, there exists a feed-forward architecture $\mathcal{A}$ with two hidden layers, $W$ weights and the Heaviside activation function such that the set $H_{\mathcal{A}}$ satisfies
\begin{align}
    \sup\limits_{f \in \mathcal{M}^d}\inf\limits_{g \in H_{\mathcal{A}}} \|f - g\|_{L^p(\lambda)} \leq 
    \left\{
    \begin{array}{l l}
         c W^{-\frac{1}{\alpha}} & \textit{ if } p(d-1) \neq d \;, \\
         c W^{-\frac{1}{d}} \log(W) & \textit{ if } p(d-1) = d \;.
    \end{array}
    \right.
\end{align}
\end{prop}

\section{Conclusion and other possible applications}
\label{sec:Barron}

We proved a general lower bound on the approximation error of $F$ by $G$ in $L^p(\mu)$ norm (Theorem~\ref{general_thm}), in terms of generic properties of $F$ and $G$ (packing number of $F$, range of $F$, fat-shattering dimension of $G$). The proof relies on VC dimension theory as in the sup norm case, but uses an additional key probabilistic argument due to Mendelson (\cite{971753}, see Proposition~\ref{result_mendelson}), solving a question raised by DeVore et al.~\cite{devore_hanin_petrova_2021}.

In Sections~\ref{section Hölder} and~\ref{section monotonic} we detailed two applications, where Corollary~\ref{main_result} yields nearly optimal approximation lower bounds in $L^p$ norm, and which correspond to two examples where the approximation rate may depend or not depend on $p$.

Theorem~\ref{general_thm} and Corollary~\ref{main_result} can be used to derive approximation lower bounds for many other cases. Corollary~\ref{main_result} only requires a (tight) lower bound on the packing number of $F$, for which approximation theory provides several examples. For instance, for the \textit{Barron space} introduced in \cite{Barron}, Petersen and Voigtlaender \cite{https://doi.org/10.48550/arxiv.2112.12555} showed a tight lower bound on the log packing number in $L^p(\lambda,[0,1]^d)$ norm, of order $\eps^{-2d/(d+2)}$. Applying Corollary~\ref{main_result}, this yields an approximation lower bound of $(LW)^{-\left(\frac{1}{2} + \frac{1}{d}\right)} \log^{-3\left(\frac{1}{2} + \frac{1}{d}\right)}(W)$ for ReLU networks (see Appendix~\ref{appendix_perspectives} in the supplement for details). Other examples of sets $F$ for which tight lower bounds on the packing number (or metric entropy) are available include: multivariate cumulative distribution functions \cite{blei07-metricEntropyHighDim}, multivariate convex functions \cite{guntuboyina13-coveringConvex}, and functions with other shape constraints \cite{groeneboom_jongbloed_2014}.

Piecewise-polynomial activation functions are not essential for the current derivation. Indeed, Theorem~\ref{general_thm} can also be applied to the case where $G$ corresponds to a neural network with other activation functions such as the sigmoid. In the sigmoid case, the pseudo-dimension is known to be at most of the order of $W^4$ (see \cite{KARPINSKI1997169,MR1741038}), which we can use to derive an approximation lower bound similar to that of Corollary 1, with a smaller right-hand side for large $W$. However, to the best of our knowledge, it is not known whether the $\mathcal{O}(W^4)$ VC bound is tight (only a lower bound of the order of $W^2$ is known), so the resulting approximation lower bound could be loose. We leave this interesting question for future work.

Theorem~\ref{general_thm} can also be applied to other approximating sets $G$, beyond classical feed-forward neural networks, as soon as a (tight) upper bound on the fat-shattering dimension of $G$ is available. For example, upper bounds were derived by \cite{wang22-vcdim-partiallyQuantizedNN} on the VC dimension of partially quantized networks, while \cite{belkin18a-RKHS} derived bounds on the fat-shattering dimension of some RKHS. Investigating such applications and whether the obtained approximation lower bounds are rate-optimal is a natural research direction for the future.

\section*{Acknowledgements} The authors would like to thank Keridwen Codet for contributing to the results of Sections~\ref{sec:monotonic-warmup} and~\ref{sec:monotonic-upperbound}.

This work has benefited from the AI Interdisciplinary Institute ANITI, which is funded by the French ``Investing for the Future – PIA3'' program under the Grant agreement ANR-19-P3IA-0004. The authors gratefully acknowledge the support of IRT Saint Exupéry and the DEEL project.\footnote{\url{https://www.deel.ai/}}

\bibliographystyle{alpha} 
\bibliography{biblio}





\newpage

\begin{center}

  \vbox{%
    \hsize\textwidth
    \linewidth\hsize
    \vskip 0.1in
          \hrule height 4pt
          \vskip 0.25in
          \vskip -\parskip%
    \centering
    {\LARGE{\bf A general approximation lower bound in $L^p$ norm, with applications to feed-forward neural networks} \\ \smallskip Supplementary Material \par}
          \vskip 0.29in
          \vskip -\parskip
          \hrule height 1pt
          \vskip 0.09in%
  }
\end{center}

\appendix
This is the appendix for ``A general approximation lower bound in $L^p$ norm, with applications to feed-forward neural networks''.

\section{Feed-forward neural networks: formal definition}
\label{sec:deffeed-forward}

In all this paper, we use the following classical graph-theoretic definitions for feed-forward neural networks given, e.g., in \cite{JMLR:v20:17-612} (with slightly different terms and notation).

A \emph{feed-forward neural network architecture} $\mathcal{A}$ of depth $L \geq 1$ is a directed acyclic graph $(V,E)$ with $d \geq 1$ nodes with in-degree $0$ (also called the \emph{input neurons}), a single node with out-degree $0$ (also called the \emph{output neuron}), and such that the longest path in the graph has length $L$. 

We define layers $\ell=0,1,\ldots,L$ recursively as follows:
\begin{itemize}
    \item layer $0$ is the set $V_0$ of all input neurons; we assume that $V_0 = \{1,\ldots,d\}$ without loss of generality.
    \item for any $\ell=1,\ldots,L$, layer $\ell$ is the set $V_\ell$ of all nodes that have one or several predecessors\footnote{A node $u \in V$ is a predecessor of another node $v \in V$ if there is a directed edge from $u$ to $v$.} in layer $\ell-1$, possibly other predecessors in layers $0,1,\ldots,\ell-2$, but no other predecessors.
\end{itemize}

Layer $L$ consists of a single node: the output neuron. Layers $1,\ldots,L-1$ are called the \emph{hidden layers} (if $L \geq 2)$. Note that skip connections are allowed, i.e., there can be connections between non-consecutive layers.

Given a feed-forward neural network architecture $\mathcal{A}$ of depth $L \geq 1$, we associate real numbers $w_e \in \R$ to all edges $e \in E$ and $w_v \in \R$ to all nodes $v \in V_1 \cup \ldots \cup V_L$. These real numbers are called \emph{weights}
(they correspond to linear coefficients and biases) and are concatenated in a \emph{weight vector} $\bw \in \R^W$, where $W = \card(E) + \sum_{\ell=1}^L \card(V_\ell)$ is the total number of weights. 

Given $\mathcal{A}$, an associated weight vector $\bw \in \R^W$, and a function $\sigma:\R \to \R$ (called \emph{activation function}), the network represents the function $g_{\bw}:\R^d \to \R$ defined recursively as follows. We write $P_v \subset V$ for the set of all predecessors of any node $v \in V$, and $w_{u \to v}$ for the weight on the edge from $u$ to $v$. The recursion from layer $\ell=0$ to layer $\ell=L$ reads: given $x = (x_1,\ldots,x_d) \in \R^d$,
\begin{itemize}
    \item each input neuron $v \in \{1,\ldots,d\}$ outputs the value $y_v := x_v$;
    \item for any $\ell=1,\ldots,L-1$, each neuron $v \in V_\ell$ outputs $y_v := \sigma\bigl(\sum_{u \in P_v} w_{u \to v} y_u + w_v\bigr)$;
    \item the unique output neuron $v \in V_L$ outputs $g_{\bw}(x) := \sum_{u \in P_v} w_{u \to v} y_u + w_v$.
\end{itemize}

Finally, we define $H_{\mathcal{A}} := \{g_{\bw}: \bw \in \R^W\}$ to be the set of all functions that can be represented by tuning all the weights assigned to the network (the dependency on the activation function $\sigma$ is not written explicitly).

\section{Main results: technical details}
\label{Appendix}

We provide technical details that were missing to establish Proposition \ref{result_mendelson}, Theorem~\ref{general_thm} and Corollary~\ref{main_result}.

\subsection{Proof of Proposition \ref{result_mendelson}}
\label{sec:rescaling}

Proposition \ref{result_mendelson} is a direct extension of \cite[Corollary~3.12]{971753} to any range $[a,b]$. We first recall this result but in slightly different terms (see the comments afterwards).

\begin{prop}[Corollary~3.12 in \cite{971753}, ``almost equivalent'' statement]
    \label{prop_mendelson_original}
    Let $G$ be a set of measurable functions from a measurable space $\cX$ to $[0,1]$, such that $\fat_\gamma(G)<+\infty$ for all $\gamma>0$. Then, for every $1 \leq p < +\infty$, there is some constant $c_p>0$ depending only on $p$ such that, for every probability measure $\mu$ on $\cX$ and every $\eps>0$,
    \begin{align*}
    \log M\left(\varepsilon, G, \|\cdot\|_{L^p(\mu)}\right) \leq c_p \fat_{\frac{\varepsilon}{32}}(G) \log^2 \left(\frac{2 \fat_{\frac{\varepsilon}{32}}(G)}{\varepsilon}\right).
    \end{align*}
\end{prop}

To be precise, \cite[Corollary~3.12]{971753} was stated a little differently. Instead of the assumption on $\fat_\gamma(G)$, there were two conditions on $G$:
(i) $G$ satisfies a weak measurability assumption such as the ``image admissible Suslin'' property, and (ii) $G$ is a uniform Glivenko-Cantelli class.
Fortunately, note that assumption (i) could easily be checked in special cases such as the setting of Corollary~\ref{main_result}, and that assumption (ii) is equivalent to $\fat_\gamma(G)<+\infty$ for all $\gamma>0$ when (i) holds and when $G$ only consists of $[0,1]$-valued functions (see
\cite{alon97-scaleSensitiveDimensions}, Theorem~2.5). The two statements are thus ``almost equivalent''. However, we stress that (i) and (ii) are not necessary (assuming $\fat_\gamma(G)<+\infty$ for all $\gamma>0$). To see why, it suffices to adapt the proof of \cite[Corollary~3.12]{971753} as follows: instead of starting from an $\varepsilon$-packing of $G$ in empirical $L^p(\mu_n)$ norm and showing that it is also an $\varepsilon'$-covering of $G$ in $L^p(\mu)$ norm, with $\varepsilon'>\varepsilon$, we can start from an $\varepsilon$-packing of $G$ in $L^p(\mu)$ norm and show that it is also an $\varepsilon$-packing of $G$ in empirical $L^p(\mu_n)$ norm for some large integer $n$ (with positive probability). This last statement directly follows from the Hoeffding inequality: no uniform law of large numbers is required, since we only need to compare empirical averages to their expectations for a finite number of bounded functions.\footnote{In passing, all occurrences of $\fat_{\frac{\varepsilon}{32}}(G)$ could be replaced with $\fat_{\frac{\varepsilon}{8}}(G)$.}

We now explain how to derive Proposition~\ref{result_mendelson} (with an arbitrary range $[a,b]$) as a straightforward consequence of Proposition~\ref{prop_mendelson_original}.

\begin{proof}[Proof (of Proposition \ref{result_mendelson}).]
\label{proof_lemma_ext_mendel}
In order to apply Proposition~\ref{prop_mendelson_original}, we reduce the problem from $[a,b]$ to $[0,1]$ by translating and rescaling every function in $G$. For $g \in G$, we define $\tilde{g}:\cX \to [0,1]$ by $\tilde{g}(x) = \frac{g(x) - a}{b-a}$, and we set
    $$
    \tilde{G} = \left\{\tilde{g} \; : \ g \in G \right\} \;.
    $$
    Note that every $\tilde{g} \in \tilde{G}$ is indeed $[0,1]$-valued.
    
We now note that translation does not affect packing numbers nor the fat-shattering dimension, while rescaling only changes the scale $\eps$ by a factor of $b-a$. More precisely, we have the following two properties:

\textbf{Property 1:} For all $u > 0$, $\fat_{\frac{u}{b-a}}(\tilde{G}) = \fat_u(G)$. \\
\textbf{Property 2:} For all $u > 0$, $M\!\left(\frac{u}{b-a}, \tilde{G}, \|\cdot\|_{L^p(\mu)}\right) = M\!\left(u, G, \|\cdot\|_{L^p(\mu)}\right)$.

Before proving the two properties (see below), we first conclude the proof of Proposition \ref{result_mendelson}. By Property~1, $\fat_{\gamma}(\tilde{G}) = \fat_{\gamma (b-a)}(G)$, which by assumption is finite for all $\gamma>0$. Since every $\tilde{g} \in \tilde{G}$ is $[0,1]$-valued, we can thus apply Proposition~\ref{prop_mendelson_original}. Using it with $\tilde{\eps} = \eps/(b-a)$, we get
\[
\log M\!\left(\tilde{\varepsilon}, \tilde{G}, \|\cdot\|_{L^p(\mu)}\right) \leq c_p \fat_{\frac{\tilde{\varepsilon}}{32}}(\tilde{G}) \log^2 \!\left(\frac{2 \fat_{\frac{\tilde{\varepsilon}}{32}}(\tilde{G})}{\tilde{\varepsilon}}\right) \;.
\]
Combining with the two equalities in Properties~1 and ~2, we obtain
\[
\log M\!\left(\eps, G, \|\cdot\|_{L^p(\mu)}\right) \leq c_p \fat_{\frac{\eps}{32}}(G) \log^2 \!\left(\frac{2 (b-a) \fat_{\frac{\eps}{32}}(G)}{\eps}\right) \;,
\]
which concludes the proof of Proposition \ref{result_mendelson}.

We now prove the two properties.
    
\textbf{Proof of Property~1.} We first show that $\fat_{\frac{u}{b-a}}(\tilde{G}) \geq \fat_u(G)$. To that end, let $S = \{x_1, \ldots, x_m\}$ and $r : S \to \mathbb{R}$ be such that for any $E \subset S$, there exists $g \in G$ such that $g(x) \geq r(x) + u$ if $x \in E$ and $g(x) \leq r(x) - u$ otherwise. Setting $\tilde{r}(x) = \frac{r(x) - a}{b-a}$, we can see that $\tilde{g}(x) \geq \tilde{r}(x) + \frac{u}{b-a}$ if $x \in E$ and $\tilde{g}(x) \leq \tilde{r}(x) - \frac{u}{b-a}$ otherwise, which proves $\fat_{\frac{u}{b-a}}(\tilde{G}) \geq \fat_u(G)$. The reverse inequality is proved similarly.
    
\textbf{Proof of Property~2.} Let $\{g_1, \ldots, g_m\}$ be a $u$-packing of $G$ in $L^p(\mu)$ norm. This means that $\|g_i - g_j\|_{L^p(\mu)}>u$ and therefore $\|\tilde{g}_i - \tilde{g}_j\|_{L^p(\mu)}>\frac{u}{b-a}$ for all $i \neq j \in \{1,\ldots,m\}$, so that $\{\tilde{g}_1, \ldots, \tilde{g}_m\} \subset \tilde{G}$ is a $\frac{u}{b-a}$-packing of $\tilde{G}$. This proves $M\bigl(\frac{u}{b-a}, \tilde{G}, \|\cdot\|_{L^p(\mu)}\bigr) \geq M\bigl(u, G, \|\cdot\|_{L^p(\mu)}\bigr)$.
The reverse inequality is proved similarly.
\end{proof}

\subsection{Clipping can only help}
\label{sec:clipping}

The next two lemmas indicate that clipping (truncature) to a known range can only help. These are key to apply Proposition~\ref{result_mendelson} in our setting.
In the sequel, for a set $G$ of functions from a set $\cX$ to $\mathbb{R}$, and for $a < b$ in $\mathbb{R}$, we denote by $G_{[a,b]}$ the set of all functions in $G$ whose values are truncated (clipped) to the segment $[a,b]$, that is, $G_{[a,b]} = \{\Tilde{g}: g \in G\}$, where $\Tilde{g}:\cX \to \R$ is given by
$$
\forall x \in \mathcal{X}, \quad
\Tilde{g}(x) = \min(\max(a,g(x)), b) \;.
$$
\begin{lemma}
\label{lemma_fatG}
Let $G$ be a set of functions defined on a set $\mathcal{X}$, and with values in $\mathbb{R}$. Let $G_{[a,b]}$ be defined as above.
Then, for any $\gamma > 0$,
\begin{align*}
&\fat_{\gamma}(G) \geq \fat_{\gamma}(G_{[a,b]}) \;.
\end{align*}
\end{lemma}

\begin{proof}
Let $\gamma > 0$. The case when $\fat_{\gamma}(G_{[a,b]}) = 0$ is straightforward. We thus assume that $\fat_{\gamma}(G_{[a,b]}) \geq 1$. 
To prove the result, we show that any subset $A$ of $X$ that is \textit{$\gamma$-shattered} by $G_{[a,b]}$ is also \textit{$\gamma$-shattered} by $G$.
Let us consider such a subset $A = \{x^1, \ldots, x^N\} \subset X$, with cardinality $N \geq 1$. Hence, there exists $\{r_1, \ldots, r_N\} \subset \mathbb{R}$ such that for any $E \subset A$, there exists $\tilde{g} \in G_{[a,b]}$ such that $\Tilde{g}(x_i) - r_i \geq \gamma$ if $x_i \in E$ and $\Tilde{g}(x_i) - r_i \leq -\gamma$ otherwise. Note that this must imply that $r_i \in ]a,b[$ for all $i = 1, \ldots, N$ (indeed, by choosing $E$ such that $x_i \in E$ or not, we have either $r_i + \gamma \leq \tilde{g}(x_i) \leq b$ or $r_i - \gamma \geq \tilde{g}(x_i) \geq a$). Now fix $i \in \{1, \ldots, N\}$ and let us assume $\tilde{g}(x_i) - r_i \geq \gamma$ (by symmetry, the reversed case $\tilde{g}(x_i) - r_i \leq -\gamma$ is treated the same way). Because $r_i > a$, this implies that $\tilde{g}(x_i) > a$ and thus $g(x_i) \geq \tilde{g}(x_i)$ (by definition of $\tilde{g}$), which entails $g(x_i) - r_i \geq \gamma$. It follows that if $G_{[a,b]}$ \textit{$\gamma$-shatters} $A$, then $G$ also \textit{$\gamma$-shatters} $A$, and the result follows.
\end{proof}
The following lemma formalizes the well-known idea that it is easier to approach a function with values in a finite range by a function with values in the same range.
\begin{lemma}
\label{lemma_supinf_G}
Let $G$ be a set of measurable functions from a measurable space $\mathcal{X}$ to $\mathbb{R}$, and let $G_{[a,b]}$ be defined as above. Assume $F$ is a set of measurable functions from $\mathcal{X}$ to $[a,b]$. Then, for any probability measure $\mu$ on $\mathcal{X}$,
$$
\sup_{f \in F}\inf_{g \in G} \|f - g\|_{L^p(\mu)} \geq \sup_{f \in F} \inf_{\tilde{g} \in G_{[a,b]}} \|f - \tilde{g}\|_{L^p(\mu)} \;.
$$
\end{lemma}

\begin{proof}
To prove the above result, it is enough to show that for any $f \in F$ and $g \in G$, the function $\tilde{g}$ is pointwise at least as close to $f$ as $g$ is, which for all $f \in F$ yields $\inf_{g \in G} \|f - g\|_{L^p(\mu)} \geq \inf_{\tilde{g} \in G_{[a,b]}} \|f - \tilde{g}\|_{L^p(\mu)}$. By definition of $G_{[a,b]}$, for any $x \in \cX$, if $g(x) \in [a,b]$, then $|f(x) - g(x)| = |f(x) - \tilde{g}(x)|$. And if $g(x) \notin [a,b]$, then $|f(x) - \tilde{g}(x)| < |f(x) - g(x)|$ since $f(x) \in [a,b]$. It follows that the discrepancy $|f - \tilde{g}|$ is everywhere bounded by $|f - g|$, and the result follows.
\end{proof}

\subsection{Missing details in the proof of Theorem \ref{general_thm}}
\label{function_study}

We provide all details that were missing to derive \eqref{main_result_proof_conclusion}, which is a direct consequence of Lemma~\ref{resolution equation en epsilon} below. We follow the convention $a P^{-\frac{1}{\alpha}}\log^{-\frac{2}{\alpha}}(P) = +\infty$ when $P=1$.

\begin{lemma}\label{resolution equation en epsilon}
Let $P \in \mathbb{N}^*$ and $c, \alpha, r>0$. There exist constants $a,\varepsilon''_1>0$ depending only on $c$, $\alpha$ and~$r$ such that, for all $\varepsilon \in (0,r)$ satisfying 
\begin{equation}\label{erivuqetbn} 
\varepsilon^{-\alpha} \leq c P \log^2\left(\frac{rP}{\varepsilon}\right) \;,
\end{equation}
we have
$$\varepsilon \geq \min\left(\varepsilon''_1,a P^{-\frac{1}{\alpha}}\log^{-\frac{2}{\alpha}}(P)\right).$$
\end{lemma}

\begin{proof}
Assume $\varepsilon \in (0,r)$ is such that \eqref{erivuqetbn} holds. To show the result, we study the function $f:(1/r,+\infty) \to \R$ defined for all $x>1/r$ by
\[
f(x) = \frac{x^{\alpha}}{\log^2(rPx)} \;.
\]
Note that \eqref{erivuqetbn} implies that $f(1/\eps) \leq c P$. For all $P \geq 2$, we set 
\begin{align}\label{epsilon_P}
\varepsilon_P = P^{-\frac{1}{\alpha}}\log^{-\frac{2}{\alpha}}(P) \;.
\end{align}
Let $P_1 \geq 2$ be such that $P_1^{\frac{1}{\alpha}}\log^{\frac{2}{\alpha}}(P_1) \geq \frac{\exp(\frac{2}{\alpha})}{r}$. For all $P \geq P_1$, we have $\frac{1}{\eps_P} \geq \frac{\exp(\frac{2}{\alpha})}{r} > 1/r$ and
$$ f\left(\frac{1}{\varepsilon_P}\right) = \frac{P \log^2(P)
}{\log^2\left(rP^{1+\frac{1}{\alpha}}\log^{\frac{2}{\alpha}}(P)\right)} \;.$$
Since 
$$ \lim_{Q\to +\infty} \frac{ \log^2(Q)
}{\log^2\left(rQ^{1+\frac{1}{\alpha}}\log^{\frac{2}{\alpha}}(Q)\right)} = \frac{1}{(1+\frac{1}{\alpha})^2} =:c_1 \;,
$$
there exists $P_2$ such that for all $Q \geq P_2$, we have $\frac{ \log^2(Q)
}{\log^2\left(rQ^{1+\frac{1}{\alpha}}\log^{\frac{2}{\alpha}}(Q)\right)} \geq \frac{c_1}{2}$ \;.

Below we distinguish the cases $P \geq \max(P_1,P_2)$ and $P < \max(P_1,P_2)$.

\textbf{1st case:} $P \geq \max(P_1,P_2)$.\\
We have $f\bigl(\frac{1}{\varepsilon_P} \bigr) \geq \frac{c_1 P}{2}$ and $P \geq \frac{1}{c} f\!\left(\frac{1}{\varepsilon} \right)$ (by \eqref{erivuqetbn}), so that $f\bigl(\frac{1}{\varepsilon_P} \bigr) \geq \frac{c_1}{2 c} f\left(\frac{1}{\varepsilon} \right)$.
We now use Lemma~\ref{proprietes fct f} below with $b=\frac{c_1}{2 c}$: setting $a := (b/2)^{1/\alpha} = (c_1/(4c))^{1/\alpha}$, there exists $x_1>\max\bigl\{\frac{1}{r}, \frac{1}{a r}\bigr\}$ depending only on $r,b,\alpha$ such that $b f(x) \geq f(a x)$ for all $x \geq x_1$.

Therefore, if $\varepsilon < \frac{1}{x_1} =: \varepsilon_1$,
then $\frac{c_1}{2 c} f\!\left(\frac{1}{\varepsilon} \right) \geq f\bigl(\frac{a}{\varepsilon}\bigr)$. Therefore $f(\frac{1}{\varepsilon_P}) \geq f\bigl(\frac{a}{\varepsilon}\bigr)$.

Recall from \eqref{epsilon_P} and $P \geq P_1$ that $\frac{1}{\varepsilon_P} \geq \frac{\exp(\frac{2}{\alpha})}{r}$.
If $ \varepsilon < \frac{a r}{\exp(\frac{2}{\alpha})}=: \varepsilon_2$, then we also have $\frac{a}{\varepsilon} \geq \frac{\exp(\frac{2}{\alpha})}{r}$.
Therefore, using Lemma \ref{proprietes fct f} again, $f(\frac{1}{\varepsilon_P}) \geq f(\frac{a}{\varepsilon})$ implies
that
$\frac{1}{\varepsilon_P} \geq \frac{a}{\varepsilon}$, that is,
$$ \varepsilon \geq a \, \varepsilon_P \;.$$

Summarizing, when $\varepsilon \in (0,r)$ satisfies \eqref{erivuqetbn} and when $ P\geq \max(P_1,P_2) $, either $\varepsilon \geq \varepsilon_1$ or $\varepsilon \geq \varepsilon_2$ or $\varepsilon \geq a \, \varepsilon_P$. Put differently,
\begin{equation}\label{qereqgr}
    \varepsilon \geq  \min(\varepsilon_1,\varepsilon_2, a \,\varepsilon_P) \;.
\end{equation}

\textbf{2nd case:} $P<\max(P_1,P_2)=:P_3$.\\
Using \eqref{erivuqetbn} and the fact that $t \mapsto c t \log^2\bigl( \frac{rt}{\varepsilon}\bigr)$ is non-decreasing on $[\eps/r,+\infty)$, together with $\eps/r \leq 1 \leq P\leq P_3$ yields $\eps^{-\alpha} \leq c P_3 \log^2(r P_3/\eps)$. This entails that, for some $\eps_3 > 0$ depending only on $\alpha, c, P_3, r$, 
\begin{equation} \label{ergibetg}
\varepsilon \geq \varepsilon_3 \;.
\end{equation}

\textbf{Conclusion:} combining the two cases, when $\varepsilon \in (0,r)$ satisfies \eqref{erivuqetbn}, whatever $P\in\NN^*$, we have \eqref{qereqgr} or \eqref{ergibetg}. Setting $\varepsilon''_1 = \min(\varepsilon_1,\varepsilon_2,\varepsilon_3)$, we obtain $$\varepsilon \geq \min\left(\varepsilon''_1, a \, P^{-\frac{1}{\alpha}}\log^{-\frac{2}{\alpha}}(P)\right).$$
(Note that this is also true in the case $P=1$, by the convention $a P^{-\frac{1}{\alpha}}\log^{-\frac{2}{\alpha}}(P) = +\infty$.) Since $\eps_1,\eps_2,\eps_3$ and $a$ only depend on $c,\alpha,r$, this concludes the proof.
\end{proof}

\begin{lemma}\label{proprietes fct f}
Let $\alpha, r > 0$ and $P \in \N^*$.
We define $f(x) = \frac{x^{\alpha}}{\log^2(rPx)}$ for all $x > 1/r$. Then:
\begin{itemize}
    \item[i)] $f$ is increasing on $I := \left[\frac{\exp(\frac{2}{\alpha})}{r}, +\infty\right)$ and $\lim_{x\to +\infty} f(x) = +\infty$.
    \item[ii)] for all $b>0$, setting $a := (b/2)^{1/\alpha}$, there exists $x_1>\max\bigl\{\frac{1}{r}, \frac{1}{a r}\bigr\}$ depending only on $r,b,\alpha$ such that,
    \[
    \forall x \geq x_1 \;, \qquad b f(x) \geq f(a x) \;.
    \]
\end{itemize}
    
\end{lemma}

\begin{proof}

\textbf{Proof of i):}
The fact that $\lim_{x\to +\infty} f(x) = +\infty$ is because $\alpha > 0$. To see why $f$ is increasing on $I$, note that
$$f'(x) = \frac{\alpha x^{\alpha-1} \log^2(rPx)-x^{\alpha}2 \log(rPx) \frac{1}{x}   }{\log^4(rPx)} = \frac{ x^{\alpha-1} \log(rPx)(\alpha \log(rPx)- 2   )}{\log^4(rPx)} \;,$$
so that $f'(x)>0$ for all $x > \frac{ \exp(\frac{2}{\alpha}) }{rP}$, and in particular for all $x > \frac{ \exp(\frac{2}{\alpha}) }{r}$ (since $P \geq 1$). This proves that  $f$ is increasing on $I$.

\textbf{Proof of ii):}
Let $b>0$ and set $a := (b/2)^{1/\alpha}$. Let $x_1>\max\bigl\{\frac{1}{r}, \frac{1}{a r}\bigr\}$ depending only on $r,b,\alpha$ such that, for all $u \geq x_1$,
\[
\frac{\log^{2}(r u)}{\log^{2}(r a u)} \leq 2 \;.
\]
(Such an $x_1$ exists since the ratio converges to $1$ as $u \to +\infty$, and we can choose $x_1$ as a function of~$r,a$ only.) Now, for all $x \geq x_1$, using the above inequality with $u = P x \geq x$ (since $P \geq 1$), we get
\begin{align*}
    \frac{f(ax)}{f(x)} = a^{\alpha} \frac{\log^{2}(rPx)}{\log^{2}(rPax)} \leq 2 a^{\alpha} = b\;,
\end{align*}
where the last equality is because $a := (b/2)^{1/\alpha}$. This proves that $b f(x) \geq f(a x)$ for all $x \geq x_1$.
\end{proof}

\subsection{Proof of Corollary~\ref{main_result}}
\label{sec:proofcorollary-main}

We first recall some definitions and two key bounds on the VC-dimension of piecewise-polynomial feed-forward neural networks, proved by \cite{Goldberg2004BoundingTV} and \cite{JMLR:v20:17-612}.

For a set $F$ of functions from $\mathcal{X}$ to $\{-1,1\}$, we say that a set $S = \{x_1\, \ldots, x_N\} \subset \mathcal{X}$ is \textit{shattered} by $F$ if for any $E \subset S$, there exists $f \in F$ satisfying for all $i = 1, \ldots, N$, $f(x_i) = 1$ if $x_i \in E$, and $f(x_i) = -1$ if $x_i \notin E$. The VC-dimension of $F$, denoted by $\VCdim(F)$, is defined as the largest number $N \geq 1$ such that there exists $S \subset \mathcal{X}$ of cardinality $N$ which is shattered by $F$ (by convention, $\VCdim(F)=0$ if no such set $S$ exists, while $\VCdim(F)=+\infty$ if there exist sets $S$ of unbounded cardinality $N$).

Let $\cB$ be any feed-forward neural network architecture of depth $L \geq 1$ with $W \geq 1$ weights, $d \geq 1$ input neurons, and $U \geq 1$ hidden or output neurons. Let $\sigma:\R \to \R$ be any piecewise-polynomial activation function on $K \geq 2$ pieces, with maximal degree $\nu \in \N$. Denote by $\sgn(H_{\cB}) = \{\sgn(g_{\bw}) : \bw \in \R^W\}$ the set of all classifiers obtained by looking at the sign of the network's output, that is, the classifiers defined by $\sgn(g_{\bw})(x) = \mathds{1}_{\{g_{\bw}(x)>0\}}$ for all $x \in \R^d$.

Goldberg and Jerrum \cite{Goldberg2004BoundingTV} showed that, for some constant $c_1'>0$ depending only on $d$, $\nu$ and $K$, the VC-dimension of $\sgn(H_{\cB})$ is bounded as follows (see also Theorem~8.7 in \cite{MR1741038}):
\begin{equation}
\label{eq:VCdimGJ04}
\VCdim(\sgn(H_{\cB})) \leq c_1'  W^2 \;.
\end{equation}

This bound was refined for piecewise-affine activation functions. Namely, Bartlett et al. \cite[Theorem~7]{JMLR:v20:17-612} proved that, if $U \geq 3$, then, for some $R \leq U + U (L-1) \nu^{L-1}$, 
\begin{align*}
    \VCdim(\sgn(H_{\cB})) & \leq
    L + \bar{L} W \log_2\Bigl(4 e (K-1) R \log_2\bigl(2 e (K-1)R\bigr)\Bigr) \;,
\end{align*}
where $\bar{L} = 1$ if $\nu = 0$, and $\bar{L} \leq L$ otherwise. Therefore, for some constants $W'_{\min} \geq 1$ and $c'_2,c'_3>0$ depending only on $d$ and $K$, we have, for all $W \geq W'_{\min}$ (which in particular implies $U \geq 3$),
\begin{align}
    \VCdim(\sgn(H_{\cB})) & \leq
    \left\{
    \begin{array}{l l}
         c_2'  LW\log(W)& \text{ if } \nu=1 \,, \\
         c_3' W\log(W) & \text{ if } \nu =0 \,.
    \end{array}
    \right. \label{eq:VCdimBHLM19}
\end{align}

We are now ready to prove Corollary~\ref{main_result} from Theorem~\ref{general_thm}.

\begin{proof}[Proof (of Corollary~\ref{main_result})]
In order to apply Theorem~\ref{general_thm}, we first bound $P := \Pdim(H_{\cA})$ from  above. The bounds \eqref{eq:VCdimGJ04} and \eqref{eq:VCdimBHLM19} were on the VC-dimension of $\sgn(H_{\cB})$, for any feed-forward neural network architecture $\cB$, while we need a bound on the pseudo-dimension. However, by a well-known trick (e.g., Theorem~$14.1$ in \cite{MR1741038}), the pseudo-dimension of $H_{\cA}$ is upper bounded by the VC-dimension of (the sign of) an augmented network architecture of depth $L$, with $d+1$ input neurons and $W+1$ weights.\footnote{This is because $\Pdim(H_{\cA}) = \VCdim\bigl(\{(x,r) \in \R^d \times \R \mapsto \mathds{1}_{\{g(x) - r >0\}}  : g \in H_{\cA}\}\bigr)$, the output neuron of $\cA$ is linear, and we allow skip connections.}
Therefore, replacing $(d,W)$ with $(d+1,W+1)$ in \eqref{eq:VCdimGJ04} and \eqref{eq:VCdimBHLM19}, we get that, for some constants $\tilde{W}_{\min} \geq 1$ and $\tilde{c_1}, \tilde{c_2}, \tilde{c_3} > 0$ depending only on $d$, $\nu$ and $K$, for all $W \geq \tilde{W}_{\min}$,
\begin{align}
    \label{eq:ub_bartlett_al_pdim}
    P \leq
    \left\{
    \begin{array}{l l}
         \tilde{c_1}  W^{2} & \text{ if } \nu \geq 2 \;, \\
         \tilde{c_2}  LW\log(W)& \text{ if } \nu=1 \;, \\
         \tilde{c_3} W\log(W) & \text{ if } \nu =0 \;.
    \end{array}
    \right.
\end{align}

Now, by Theorem \ref{general_thm}, we have, for some constants $c_1,\eps_1>0$ depending only on $b-a$, $p$, $c_0$, $\varepsilon_0$, $\alpha$,
\begin{align}
    \label{eq:main_thm_rappel}
    \sup_{f \in F} \inf_{g \in H_{\mathcal{A}}} \|f - g\|_{L^p(\mu)} \geq \min\left\{\varepsilon_1, \ c_1 P^{-\frac{1}{\alpha}}\log^{-\frac{2}{\alpha}}(P)\right\} \;.
\end{align}
Noting that $P \mapsto \min\left\{\varepsilon_1, \ c_1 P^{-\frac{1}{\alpha}}\log^{-\frac{2}{\alpha}}(P)\right\}$ is non-increasing and plugging~(\ref{eq:ub_bartlett_al_pdim}) into~(\ref{eq:main_thm_rappel}), we get, for $W \geq W_{\min}$,
\begin{align*}
    \sup_{f \in F} \inf_{g \in {H}_\mathcal{A}} \|f - g\|_{L^p(\mu)} \geq \min \left\{\varepsilon_1,
    \left(
    \begin{array}{l l}
         c_4  W^{-\frac{2}{\alpha}} \log^{-\frac{2}{\alpha}}(W^2) & \text{ if } \nu \geq 2 \\
         c_5  (LW\log(W))^{-\frac{1}{\alpha}} \log^{-\frac{2}{\alpha}}(LW\log(W)) & \text{ if } \nu=1 \\
         c_6  (W\log(W))^{-\frac{1}{\alpha}} \log^{-\frac{2}{\alpha}}(W\log(W)) & \text{ if } \nu=0
    \end{array}
    \right)
    \right\}
\end{align*}
for some constants $W_{\min} \geq 1$ and $c_4, c_5, c_6>0$
depending only on $d$, $\nu$, $K$, $b-a$, $p$, $c_0$, $\varepsilon_0$ and~$\alpha$. Taking $W_{\min}$ large enough, the first term $\varepsilon_1$ is always larger than the second term in the above minimum, and the logarithmic terms $\log(W\log(W))$ and $\log(LW\log(W))$ can be upper bounded by a constant times $\log(W)$ (since $L \leq W)$. Rearranging concludes the proof.
\end{proof}

\section{Earlier works: two other lower bound proof strategies}
\label{sec:relatedworks}

Approximation lower bounds in a sense similar to ours have been obtained in other recent works. In the purpose of highlighting the differences between the different approaches, we describe the lower bound proof strategies of Yarotsky \cite{yarotsky2017error} and of Petersen and Voigtlaender \cite{PETERSEN2018296}.

\subsection{Approximation in sup norm of Sobolev unit balls with ReLU networks \cite{yarotsky2017error}}

Recall that the Sobolev space
$\mathcal{W}^{n,\infty}([0,1]^d)$ is defined as the set of functions on $[0,1]^d$ lying in $L^\infty$ along with all their weak derivatives up to order $n$. We equip this space with the norm
$$
\|f\|_{\mathcal{W}^{n,\infty}([0,1]^d)} = \max_{\textbf{n} \in \mathbb{N}^d : |\textbf{n}| \leq n} \esssup_{x \in [0,1]^d} |D^{\textbf{n}}f(x)|,
$$
and we let $F_{n,d}$ be the unit ball of this space.

We first state the sup norm lower bound and then we give a synthesized version of the proof.

\begin{prop}[\cite{yarotsky2017error}]
    \label{prop_lb_yarotsky_2017}
    There exists positive constants $W_{\min}, c> 0$ such that for any feed-forward neural network with architecture $\mathcal{A}$, ReLU activation and $W \geq W_{\min}$ weights,
    $$
    \sup_{f \in F_{n,d}} \inf_{g \in H_{\mathcal{A}}} \|f - g\|_{\infty} \geq c W^{-\frac{2n}{d}}.
    $$
\end{prop}

Details aside, the proof reads as follows. The author assumes that $H_\mathcal{A}$ approximates $F_{n,d}$ with error $\varepsilon$. Fixing $N = c_{n,d}(3\varepsilon)^{-1/n}$ for a properly chosen constant $c_{n,d} > 0$, he constructs a set of functions in $F_{n,d}$ that can shatter a grid of $N^d$ points $x_1, \ldots, x_{N^d}$ evenly distributed over $[0,1]^d$. The assumption that $H_\mathcal{A}$ approximates $F_{n,d}$ in sup norm with error $\varepsilon$ allows to conclude that $H_\mathcal{A}$ also shatters $\{x_1, \ldots, x_{N^d}\}$, and hence, $\VCdim(H_\mathcal{A}) \geq N^d = c'_{n,d} \varepsilon^{-\frac{d}{n}}$, for a properly chosen constant $c'_{n,d} > 0$. The author concludes using the upper bound on $\VCdim(H_\mathcal{A})$ with respect to $W$ from \cite{MR1741038} which yields $\VCdim(H_\mathcal{A}) \leq c' W^2$ for some constant $c'$.

It is worth stressing that in this proof, it is paramount to assume that $H_\mathcal{A}$ approximates $F_{n,d}$ in sup norm, rather than any $L^p$ norm with $p < +\infty$. The reason is that only this choice of norm allows to bound the discrepancy between $f \in F_{n,d}$ and $g_f \in H_\mathcal{A}$ chosen optimally with respect to $f$ at any chosen points. Our proof strategy relying on Proposition \ref{result_mendelson} allows to circumvent this issue by relating the pseudo-dimension to the metric entropy with respect to any $L^p$ norm, $1 \leq p < +\infty$.

\subsection{Approximation in $L^p$ norm of \textit{Horizon functions} with quantized networks \cite{PETERSEN2018296}}

The authors study \textit{quantized} neural networks, that is, networks with weights constrained to be representable with a fixed number of bits. They obtain a lower bound on the minimal number of weights in a quantized neural network that can approximate a set of \textit{Horizon functions} in $L^p$ norm, $p > 0$, with error $\varepsilon > 0$. This lower bound is easily invertible to a bound on the approximation error and is thus comparable to the results we obtain in this paper.

Textually, the authors introduce the set of horizon functions as follows:
``These are $\{0, 1\}$-valued functions with a jump along a hypersurface and such that the jump surface is
the graph of a smooth function'' \cite{PETERSEN2018296}. Denoting by $H$ the indicator function of the set $[0,+\infty) \times \mathbb{R}^{d-1}$, the set of horizon functions reads as
\begin{align*}
    \mathcal{HF}_{\beta, d, B} = \Biggl\{ & f \circ T \in L^{\infty}\left(\left[-\frac{1}{2}, \frac{1}{2}\right]^d\right) \ : \\
    & \quad f(x) = H(x_1 + \gamma(x_2, \ldots, x_d), x_2, \ldots, x_d), \gamma \in \mathcal{F}_{\beta, d-1, B}, \ T \in \Pi(d, \mathbb{R}) \Biggr\} \;,
\end{align*}

where $\mathcal{F}_{\beta, d-1, B}$ denotes the set of Hölder functions over $\left[-1/2, 1/2\right]^{d-1}$ whith smoothness parameter $\beta$ and with norm $\|.\|_{\mathcal{C}^{n,\alpha}}$ bounded by $B$ (see Section \ref{section Hölder}), and $\Pi(d, \mathbb{R})$ denotes the group of $d$-dimensional permutation matrices.

In the following, for any nonzero integer $K$ and any neural network architecture $\mathcal{A}$, we denote by $H_\mathcal{A}^K \subset H_\mathcal{A}$ the set of $K$-quantized functions in $H_\mathcal{A}$; namely, the functions in $H_\mathcal{A}$ with weights representable using at most $K$ bits. The lower bound in \cite{PETERSEN2018296} (Theorem $4.2$) reads as follow:

\begin{prop}[\cite{PETERSEN2018296}]
    \label{prop_lb_petersen}
    Let $d \geq 2$. Let $p, \beta, B, c_0 > 0$ and let $\sigma : \mathbb{R} \rightarrow \mathbb{R}$ be such that $\sigma(0) = 0$. There exist positive constants $\varepsilon_0, c> 0$ depending only on $d, p, \beta, B$ and $c_0$ such that, for any $\varepsilon \leq \varepsilon_0$, setting $K = \lceil c_0 \log(1/\varepsilon)\rceil$, for any feed-forward neural network architecture $\mathcal{A}$ with $W$ weights and activation $\sigma$ such that $H_\mathcal{A}^K$ approximates $\mathcal{HF}_{\beta, d, B}$ in $L^p$ norm with error less than $\varepsilon$, we have
    $$
    W \geq c \varepsilon^{-\frac{p(d-1)}{\beta}} \log^{-1} (1/\varepsilon).
    $$
\end{prop}

The proof of this result is based on a lemma giving a lower bound on the minimal number of bits $\ell$ necessary for a binary encoder-decoder pair to achieve an error less than $\varepsilon > 0$ in approximating $\mathcal{HF} := \mathcal{HF}_{\beta, d, B}$ in $L^p$ norm. Formally, given an integer $\ell > 0$, a binary encoder $E^{\ell} : \mathcal{HF}~\rightarrow~\{0,1\}^\ell$ and given a decoder $D^\ell : \{0,1\}^\ell \rightarrow \mathcal{HF}$, one can measure an approximation error
$$
\sup_{f \in \mathcal{HF}}\|f - D^\ell(E^\ell(f))\|_{L^p},
$$
\noindent which quantifies the loss of information due to the encoding $E^\ell$. Clearly, for an optimal choice of encoder, one can reduce this loss of information by increasing $\ell$. In particular, for $\varepsilon > 0$, it is possible to estimate
$$
\ell_{\varepsilon} = \min \left\{\ell > 0 \ : \inf_{E^\ell, D^\ell} \sup_{f \in \mathcal{HF}} \|f - D^\ell(E^\ell(f))\|_{L^p} \leq \varepsilon \right\},
$$
with the convention that $\ell_{\varepsilon} = \infty$ if the above set is empty. The authors show in their Lemma B.3 that for $\varepsilon$ small enough (smaller than some $\varepsilon_0 > 0$), it holds that
\begin{align}
    \label{eq_bit_length}
    \ell_{\varepsilon} \geq c \varepsilon^{-\frac{p(d-1)}{\beta}}
\end{align}
for some constant $c > 0$ depending only on $d, p, \beta$ and $B$. In other words, one can not achieve a loss of information smaller than $\varepsilon$ by encoding functions in $\mathcal{HF}$ over less than $c \varepsilon^{-\frac{p(d-1)}{\beta}}$ bits.

The rest of the proof consists in showing that for an integer $K > 0$, given a neural network architecture $\mathcal{A}$ with $W$ weight that can approximate $\mathcal{HF}$ in $L^p$ norm with error less than $\varepsilon > 0$, one can encode exactly (without loss of information, and for a given activation function) any function in $H_\mathcal{A}^K$ over a string of $\ell = c_1 W (K + \lceil \log_2 W \rceil)$ bits. This generates a natural encoder-decoder system where any function $f \in \mathcal{HF}$  is encoded as the bit string of length $\ell$ associated to $g_f \in H_\mathcal{A}^K$ chosen to approximate $f$. It remains to observe that if we fix $K$, this automatically yields a lower bound on $\ell$ using inequality~(\ref{eq_bit_length}), and thus on $W$ by expressing $W$ through $\ell$ and $K$.

\textbf{Remark.} The authors in \cite{PETERSEN2018296} study the neural network approximation in a setting slightly different from ours, since they focus on the approximation by quantized neural networks. This partly explains why their proof strategy differs from ours. However, it is worth pointing out that the proof of their lower bound on the minimal number of bits required to accurately encode a function in $\mathcal{HF}$ relies on a lower bound of the packing number of $\mathcal{HF}$, just like the lower bound of the packing number of the set to approximate is key in our proof strategy. An interesting question for the future would be to see whether our general lower bound (Theorem~\ref{general_thm}) yields lower bounds of the same order as those in \cite{PETERSEN2018296} for quantized neural networks.

\section{Hölder balls}

\subsection{Proof of Lemma \ref{lemme lower bound metric entropy Hölder}}
\label{proof of lemme lower bound metric entropy Hölder}

Though not necessarily stated this way, many arguments below are classical (see, e.g., Theorem~3.2 by \cite{GyKoKrWa-02-DistributionFreeNonparametric} with a similar construction for lower bounds in nonparametric regression).

Let $N \in \mathbb{N}^*$. For $\textbf{m} = (m_1, \ldots, m_d) \in \{0, \ldots, N-1\}^d$, we let $x_{\textbf{m}} := \frac{1}{N}(m_1 + 1/2, \ldots, m_d + 1/2)$ and we denote by $C_{\textbf{m}}$ the cube of side-length $\frac{1}{N}$ centered at $x_{\textbf{m}}$, with sides parallel to the axes. We see that the $N^d$ cubes
$C_{\textbf{m}}$ decompose the cube $[0,1]^d$ in smaller parts which, up to negligible sets which will not be problematic, form a partition of $[0,1]^d$. We will use this decomposition to construct a packing of $F_{s,d}$. Denoting $\|\cdot\|$ the sup norm in $\RR^d$, we define the $C^\infty$ test function $\phi : \mathbb{R}^d \rightarrow \mathbb{R}$ by:
$$
\phi(x) = \exp\left(-\frac{\|x\|^2}{1 - \|x\|^2}\right),
$$
for any $x\in\RR^d$ such that $\|x\| < 1$, and $\phi(x) = 0$ otherwise. Recalling that $n\in\NN$ and $\alpha\in (0,1]$ are such that $s = n + \alpha$, and since all the high-order partial derivatives of $\phi$ are uniformly bounded on $[0,1]^d$, $\|\phi\|_{\mathcal{C}^{n,\alpha}}$ is thus finite and is nonzero. 

Let $c_s = \frac{1}{2} (2N)^{-s} \|\phi\|_{\mathcal{C}^{n,\alpha}}^{-1}$ and consider, for any tensor of signs $\sigma = (\sigma_\mbf)_{\mbf\in\{0,\cdots,N-1\}^d} \in \{-1,1\}^{N^d}$, the function $f_\sigma$ defined as follows:
\[\label{define_f_sigma}
f_\sigma(x) = c_s \sum_{\textbf{m} \in \{0, \ldots, N-1\}^d}  \sigma_{\textbf{m}} \phi\left(2N(x - x_{\textbf{m}})\right),
\]
for all $x\in[0,1]^d$. There are $2^{N^d}$ different functions $f_\sigma$. 

Let us prove that, for all $\sigma \in \{-1,1\}^{N^d}$, $f_\sigma \in F_{s,d}$. To do so, we study the constituents of $\|f_\sigma\|_{\mathcal{C}^{n,\alpha}}$ separately and show that they are all bounded by $1$. For $\mbf\in\{0,\cdots,N-1\}^d$, we define the function $g_{\mbf}(x) = c_s \sigma_{\textbf{m}} \phi\left(2N(x - x_{\textbf{m}})\right)$. Note that because $\phi$ vanishes outside $(-1,1)^d$, we have that $g_{\mbf}$ vanishes everywhere outside the interior of $C_{\mbf}$, and the same holds for $D^{\nbf} g_{\mbf}$ for all $\nbf \in \mathbb{N}^d$ such that $|\nbf| \leq n$. For any such $\textbf{n}$, we have
\[ \|D^\nbf g_{\mbf}\|_\infty = c_s (2 N)^{|\nbf|} \|D^\nbf \phi\|_\infty \leq c_s (2 N)^s \|\phi\|_{\mathcal{C}^{n, \alpha}} \leq \frac{1}{2}.
\]
Therefore,
\[\max_{\textbf{n}:|\textbf{n}| \leq n}\|D^{\textbf{n}}f_\sigma\|_{\infty} \leq 1.
\]
Now for any $\textbf{n}\in \NN^d$ such that $|\textbf{n}| = n$, any $x,y \in [0,1]^d$, we have
\begin{align*}
    \frac{|D^{\textbf{n}}f_{\sigma}(x) - D^{\textbf{n}}f_{\sigma}(y)|}{\|x - y\|^{\alpha}_2}
    &= \frac{|D^{\textbf{n}}g_{\mbf}(x) - D^{\textbf{n}}g_{\textbf{m}'}(y)|}{\|x-y\|^{\alpha}_2},
\end{align*}
\noindent where $x \in C_{\mbf}$ and $y \in C_{\mbf'}$ for some multi-indexes $\mbf$ and $\mbf'$. We have to distinguish between the cases $\mbf = \mbf'$ and $\mbf \neq \mbf'$. In the former case, we have
\begin{align*}
    \frac{|D^{\textbf{n}}f_{\sigma}(x) - D^{\textbf{n}}f_{\sigma}(y)|}{\|x - y\|^{\alpha}_2}
    &= c_s (2 N)^{n+\alpha} \frac{|D^{\textbf{n}}\phi(2 N(x-x_{\textbf{m}})) - D^{\textbf{n}}\phi(2 N(y - x_{\textbf{m}}))|}{\|2 N(x-x_{\textbf{m}}) - 2 N(y-x_{\textbf{m}})\|^{\alpha}_2} \\
    &= c_s (2 N)^{s} \frac{|D^{\textbf{n}}\phi(x') - D^{\textbf{n}}\phi(y')|}{\|x' - y'\|^{\alpha}_2}\\
    &\leq c_s (2 N)^{s} \|\phi\|_{\mathcal{C}^{n,\alpha}} = \frac{1}{2},
\end{align*}
where at the second line, we used the changes of variables $x' = 2N(x - x_{\mbf})$ and $y' = 2N(y - x_{\mbf})$. In the case $\mbf = \mbf'$ ($x$ and $y$ belong to the same cube), we thus have
$$
\frac{|D^{\textbf{n}}f_{\sigma}(x) - D^{\textbf{n}}f_{\sigma}(y)|}{\|x - y\|^{\alpha}_2} \leq 1.
$$
In the case $\mbf \neq \mbf'$, observe that we have
\begin{align}
    \label{eq:Dn_holder}
    |D^{\nbf} g_{\mbf}(x) - D^{\nbf} g_{\mbf'}(y)|
    &\leq 2 \max \{|D^{\nbf} g_{\mbf}(x)|, |D^{\nbf} g_{\mbf'}(y)|\}.
\end{align}
Besides, recall that $D^{\nbf} g_{\mbf}$ and $D^{\nbf} g_{\mbf'}$ both vanish outside of the interiors of $C_{\mbf}$ and $C_{\mbf'}$ respectively. We can thus rewrite (\ref{eq:Dn_holder}) as
\begin{align*}
    |D^{\nbf} g_{\mbf}(x) - D^{\nbf} g_{\mbf'}(y)|
    &\leq 2 \max \{|D^{\nbf} g_{\mbf}(x) - D^{\nbf} g_{\mbf}(y)|, |D^{\nbf} g_{\mbf'}(x) - D^{\nbf} g_{\mbf'}(y)|\} \\
    &\leq 2 c_s (2N)^n \max \{|D^{\nbf}\phi(2N(x- x_{\mbf})) - D^{\nbf}\phi(2N(y- x_{\mbf}))|, \\
    &\qquad \qquad \qquad \qquad
    |D^{\nbf}\phi(2N(y- x_{\mbf'})) - D^{\nbf}\phi(2N(y- y_{\mbf'}))|\}.
\end{align*}
This entails
\begin{align*}
    \frac{|D^{\textbf{n}}f_{\sigma}(x) - D^{\textbf{n}}f_{\sigma}(y)|}{\|x - y\|^{\alpha}_2}
    &\leq c_s 2 (2N)^{s} \max \left\{ \frac{|D^{\textbf{n}}\phi(x') - D^{\textbf{n}}\phi(y')|}{\|x' - y'\|^{\alpha}_2}, \frac{|D^{\textbf{n}}\phi(x'') - D^{\textbf{n}}\phi(y'')|}{\|x'' - y''\|^{\alpha}_2} \right\} \\
    &\leq c_s 2 (2N)^{s} \|\phi\|_{\mathcal{C}^{n,\alpha}} = 1,
\end{align*}
where $x' = 2N(x - x_{\mbf})$ and $y' = 2N(y - x_{\mbf})$, and $x'' = 2N(x - x_{\mbf'})$ and $y'' = 2N(y - x_{\mbf'})$.

Summarizing, we showed that for all $\sigma \in \{-1,1\}^{N^d}$
$$\max_{\nbf: |\nbf| \leq n} \|D^{\nbf} f_{\sigma}\|_{\infty} \leq 1 \qquad \mbox{ and }\qquad\max_{\nbf: |\nbf|=n} \sup_{x \neq y} \frac{|D^{\nbf} f_\sigma(x) - D^{\nbf} f_\sigma(y)|}{\|x - y\|_{2}^\alpha} \leq 1.$$ 
We conclude that for all $\sigma \in \{-1,1\}^{N^d}$
\[\|f_\sigma\|_{\mathcal{C}^{n,\alpha}} \leq 1,
\]
and therefore $\{f_\sigma : \sigma \in \{-1,1\}^{N^d} \} \subset F_{s,d}$.

Let us now evaluate the distance between distinct elements of $\{f_\sigma : \sigma \in \{-1,1\}^{N^d} \}$. Let $\sigma^1$, $\sigma^2 \in \{-1,1\}^{N^d}$, with $\sigma^1\neq\sigma^2$, and let $\textbf{m} \in \{0, \ldots, N-1\}^{d}$ be  such that $\sigma^1_{\textbf{m}} = -\sigma^2_{\textbf{m}}$. Let us estimate $\Delta_p$ the $L^p(\lambda)$ discrepancy between $f_{\sigma^1}$ and $f_{\sigma^2}$ on the cube $C_{\textbf{m}}$, that is
\begin{align*}
    \Delta_p^p
    &= \int_{C_{\textbf{m}}} |f_{\sigma^1}(x)  - f_{\sigma^2}(x)|^p \dd x \\
    &= 2^p c_s^p \int_{C_{\textbf{m}}}  |\phi \left(2N(x - x_\textbf{m})\right)|^p \dd x \\
    &= 2^{p} c_s^p (2N)^{-d} \|\phi\|_{L^p(\lambda)}^p.
\end{align*}
It remains to find a subset among the functions $f_\sigma$ such that any two functions of this set differ on a significant number of cubes $C_{\textbf{m}}$. According to the Varshamov-Gilbert Lemma \cite{Yu1997AssouadFA}, there exists $\Gamma \subset \{-1,1\}^{N^d}$ with cardinal at least $\exp(N^d/8)$ such that for any $\sigma^1, \sigma^2 \in \Gamma$, such that $\sigma^1 \neq \sigma ^2$, $\sigma^1$ and $\sigma^2$ differ on at least one fourth of their coordinates; i.e., $\sum_{k=1}^{N^d} \mathds{1}_{\sigma^1_k \neq \sigma^2_k} \geq \frac{N^d}{4}$. We thus fix such a set $\Gamma \subset \{-1, 1\}^{N^d}$. For any $\sigma^1$, $\sigma^2\in\Gamma$, with $\sigma^1\neq \sigma^2$,
\begin{eqnarray*}
    \|f_{\sigma^1} - f_{\sigma^2}\|^p_{L^p(\lambda)}
    &= & \sum_{\textbf{m} : \sigma^1_{\textbf{m}} \neq \sigma^2_{\textbf{m}}} \int_{C_{\textbf{m}}} |f_{\sigma^1}(x) - f_{\sigma^2}(x)|^p \dd x \\
    &\geq &\frac{N^d}{4} \Delta_p^p = \frac{2^{p-d} c^p_s}{4}\|\phi\|_{L^p(\lambda)}^p.
\end{eqnarray*}
Finally, recalling the definition of $c_s$, we have for any $\sigma^1$, $\sigma^2\in\Gamma$, with $\sigma^1\neq \sigma^2$,
\[    \|f_{\sigma^1} - f_{\sigma^2}\|_{L^p(\lambda)}
    \geq 2^{1-\frac{d+2}{p}} \frac{1}{2} (2N)^{-s} \|\phi\|_{\mathcal{C}^{n,\alpha}}^{-1} \|\phi\|_{L^p(\lambda)} = c N^{-s},
\]
where $c =2^{-s-\frac{d+2}{p}}  \frac{\|\phi\|_{L^p(\lambda)}}{\|\phi\|_{\mathcal{C}^{n,\alpha}}}$.

It follows that $\{f_\sigma \ : \ \sigma \in \Gamma\}$ is a $c N^{-s}$-packing of $F_{s,d}$. Given the lower bound on the size of $\Gamma$, this implies
$$ M\left(c N^{-s}, F_{s,d}, \|\cdot\|_{L^p(\lambda)}\right) \geq \exp(N^d/8),
$$
for all $N \in \NN^*$. 

Set $\varepsilon_0=c$ and $c_0 =2^{-d} c^\frac{d}{s}/8$. Consider $\varepsilon >0$, with $\varepsilon \leq \varepsilon_0$. To conclude the proof, we need to show that \eqref{lemme2_eq} holds for $\varepsilon$. To do so, we consider $N$: the smallest integer such that $c N^{-s} \geq \varepsilon \geq c (2N)^{-s}$. This $N\in\NN^*$ exists because $0<\varepsilon\leq\varepsilon_0=c$ and $s>0$. On one side, we have
\[M\left(\varepsilon, F_{s,d}, \|\cdot\|_{L^p(\lambda)}\right) \geq M\left(c N^{-s}, F_{s,d}, \|\cdot\|_{L^p(\lambda)}\right),
\]
and on the other side, since $2N\geq c^\frac{1}{s} \varepsilon^{-\frac{1}{s}}$,
\[\exp(N^d/8) \geq \exp(2^{-d} c^\frac{d}{s} \varepsilon^{-\frac{d}{s}} /8) = \exp(c_0\varepsilon^{-\frac{d}{s}}) .
\]
Combining the last three inequalities, we finally obtain
$$
\log M\left(\varepsilon, F_{s,d}, \|\cdot\|_{L^p(\lambda)}\right) \geq c_0 \varepsilon^{-d/s},
$$
for all $0 < \varepsilon \leq \varepsilon_0$.


\section{Monotonic functions}

This section contains the proofs of the results stated in Section \ref{section monotonic}. More precisely, in Section \ref{preuve borne sup monotone} we provide the proof of  Proposition \ref{borne sup monotone} and in Section \ref{preuve borne inf norme infinie} we provide the proof of   Proposition \ref{borne inf norme infinie}.

\subsection{Proof of Proposition \ref{borne sup monotone}}\label{preuve borne sup monotone}
The section contains two sub-sections. In the first sub-section, we provide a proposition on the representation of piecewise-constant functions with Heaviside neural-networks. Section \ref{qerouneqvibet} contains the main part of the proof of Proposition \ref{borne sup monotone}.

\subsubsection{Representing piecewise-constant functions with Heaviside neural networks}
We first describe a neural network architecture which, with the Heaviside activation function, is  able to represent functions that are piecewise-constant on cubes.
\begin{prop}\label{prop:2layers}
  Let $d\in\mathbb N^*$, $M\in\mathbb N^*$.
  There exists an architecture $\mathcal{A}$ with two-hidden layers, $2(d+1)^2M$ weights and the Heaviside activation function, such that for any $(\alpha_i)_{1\leq i\leq M}\in\mathbb R^M$, any collection  $(\mathcal C_i)_{1\leq i\leq M}$ of mutually disjoint hypercubes of $\RR^d$  the function $\tilde f\colon\RR^d\rightarrow[0,1]$ defined by
  \[\forall x\in\RR^d,\quad \tilde f(x)=\sum_{i= 1}^M\alpha_i\mathds1_{\mathcal C_i}(x)
  \]
  satisfies $\tilde{f} \in H_{\mathcal{A}}$.
\end{prop}

\begin{proof}
  Define $\sigma\colon\mathbb R\rightarrow\mathbb R$ by $\sigma(x)=\mathds1_{x\geq0}$ for all $x\in\mathbb R$.

  Let $i\in\{1,\ldots,M\}$. The cube $\mathcal C_{i}$ has $2d$ faces. These faces are supported by hyperplanes whose equations are of the form $\left<\w,x\right>+b=0$, with $\w\in\mathbb R^d$ and $b\in\mathbb R$. We allow the faces to belong to the cube or not. To distinguish them, we denote $J_i\in\{1,\ldots,2d\}$ the number of faces that belong to $\mathcal C_i$. We index the $J_i$ faces that belong to the cube from $1$ to $J_i$, and the other faces from $J_i+1$ to $2d$. Thus, for all $i\in\{1,\ldots,M\}$ and all $j\in\{1,\ldots,2d\}$, there exist $\w_j^{i}\in\mathbb R^d,b_j^{i}\in\mathbb R$ such that
    \[
      \mathcal C_i=\bigcap_{j=1}^{J_i}\{x\in\mathbb R^d\colon\left<\w_j^{i},x\right>+b_j^{i}\geq0\}\ \cap\ \bigcap_{j=J_i+1}^{2d}\{x\in\mathbb R^d\colon\left<\w_j^{i},x\right>+b_j^{i}>0\}\;.
    \]
  We rewrite:
  \begin{equation} \label{zqrmoinbt}
  \mathcal C_i=\left\{x\in\mathbb R^d\colon\sum_{j=1}^{J_i}\mathds1_{\{\left<\w_j^{i},x\right>+b_j^{i}\geq0\}}+\sum_{j=J_i+1}^{2d}\mathds1_{\{\left<\w_j^{i},x\right>+b_j^{i}>0\}}\ \geq \ 2d\right\}.
  \end{equation}
  Denoting for all $i\in\{1,\ldots,M\}$ and all $j\in\{1,\ldots,2d\}$ and for all $x\in \RR^d$,

    \[
       p^i_j(x)=\begin{cases}
                              \sigma\left(\left<\w_j^{i},x\right>+b_j^{i}\right) &\text{if } j\leq J_i\\
                              1-\sigma\left(-\left<\w_j^{i},x\right>-b_j^{i}\right)&\text{otherwise},
                            \end{cases}
    \]
    we have, see Figure \ref{fig:perceptronsdim2} and \eqref{zqrmoinbt}, for all $x\in\mathbb R^d$
    \begin{align*}
      \mathds1_{\mathcal C_{i}}(x)&=\begin{cases}
                                       1\quad\text{if $\sum_{j=1}^{2d}p^i_j(x)\geq2d$,}\\
                                       0\quad\text{otherwise},
                                     \end{cases}\\
      &=\sigma\left(\sum_{j=1}^{2d}p^i_j(x)-2d\right).
    \end{align*}
  \begin{figure}
  \begin{center}
   \begin{tikzpicture}[y=1cm,scale=1]
      \draw (1,0) -- (1,4);
      \draw (3,0) -- (3,4);
      \draw (0,1) -- (4,1);
      \draw (0,3) -- (4,3);
      \node at (0.5,0.5) {2};
      \node at (0.5,2) {3};
      \node at (0.5,3.5) {2};
      \node at (2,0.5) {3};
      \node at (3.5,0.5) {2};
      \node at (2,2) {4};
      \node at (2,3.5) {3};
      \node at (3.5,3.5) {2};
      \node at (3.5,2) {3};
      \draw [->] (1,2) -- (1.5,2);
      \node[below] at (1.3,2) {$\w^i_j$};
      \draw [->] (3,2) -- (2.5,2);
      \draw [->] (2,3) -- (2,2.5);
      \draw [->] (2,1) -- (2,1.5);
    \end{tikzpicture}
    \caption{\label{fig:perceptronsdim2}Values of the sum of the perceptrons $p^i_j(x)$ around a hypercube $\mathcal C_{i}$ in dimension 2.}
  \end{center}
  \end{figure}
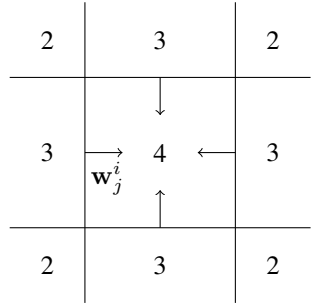
  Since the hypercubes are mutually disjoints,  for all $x\in[0,1]^d$, we have
  \begin{align}
    \tilde f(x)&=\sum_{i=1}^ M\alpha_i\sigma\left(\sum_{j=1}^{J_i}\sigma\left(\left<\w_j^{i},x\right>+b_j^{i}\right)+\sum_{j=J_i+1}^{2d}\left(1-\sigma\left(-\left<\w_j^{i},x\right>-b_j^{i}\right)\right)\ - \ 2d\right) \nonumber\\
    &=\sum_{i=1}^ M\alpha_i\sigma\left(\sum_{j=1}^{2d}\varepsilon^i_{j}\sigma\left(\left<\tilde \w_j^{i},x\right>+\tilde b_j^{i}\right)-J_i\right), \label{etgipuhqtbe}
  \end{align}
  where 
  $$\varepsilon^i_{j}=\left\{\begin{array}{ll}
  +1 & \mbox{if } j\leq J_i \\
  -1 & \mbox{otherwise,}
  \end{array}\right.
  \qquad
  \tilde \w^i_{j}=\left\{\begin{array}{ll}
   \w^i_{j}& \mbox{if } j\leq J_i \\
  -\w^i_{j} & \mbox{otherwise,}
  \end{array}\right.
  \qquad
    \tilde b^i_{j}=\left\{\begin{array}{ll}
   b^i_{j}& \mbox{if } j\leq J_i \\
  -b^i_{j} & \mbox{otherwise.}
  \end{array}\right.
  $$
  Equation \eqref{etgipuhqtbe} is the action of the Heaviside neural network with two hidden layers whose architecture is on Figure~\ref{fig:neuralnetwork2layers}.
  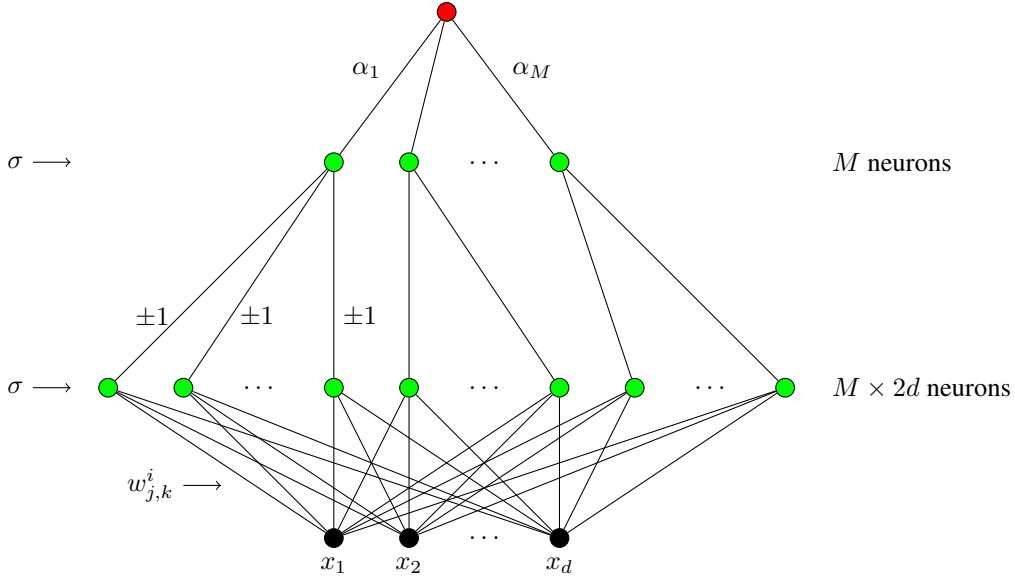
\begin{figure}
    \begin{center}
      \begin{tikzpicture}
        \tikzstyle{DISTANCE}=[outer sep=4pt]
        \foreach\NODE in {
          1-1,1-2,1-4,%
          2-1,2-2,2-4,2-5,2-7,2-8,2-10,%
          3-1,3-2,3-4,%
          4-1%
        } {
          \expandafter\gdef\csname NODE-\NODE\endcsname{true}%
        }
        \foreach \LEVEL/\NUMBER/\Y/\COLOR in {1/4/0/black,2/10/2/green,3/4/5/green,4/1/7/red}{
          \foreach \K in {1,...,\NUMBER}{
            \pgfmathsetmacro\X{\K-0.5*\NUMBER+0.5}
            \ifcsname NODE-\LEVEL-\K\endcsname
              \node[draw,circle,fill=\COLOR,inner sep=2.5pt] (N-\LEVEL-\K) at (\X,\Y) {};
            \else
              \node (N-\LEVEL-\K) at (\X,\Y) {$\ldots$};
            \fi
          }
        }
        \foreach \BEGIN/\END in {%
          1-1/{2-1//,2-2//,2-4//,2-5//,2-7//,2-8//,2-10//},1-2/{2-1//,2-2//,2-4//,2-5//,2-7//,2-8//,2-10//},1-4/{2-1//,2-2//,2-4//,2-5//,2-7//,2-8//,2-10//},%
          2-5/{3-2//},2-7/{3-2//},2-8/{3-4//},2-10/{3-4//},%
          3-1/{4-1/above left/$\alpha_1$},3-2/{4-1//},3-4/{4-1/above right/$\alpha_M$}%
        } {
          \foreach \NODE/\OPTIONS/\LABEL in \END {
            \draw (N-\BEGIN) -- (N-\NODE) node[midway,\OPTIONS] {\LABEL};
          }
        }
        \draw (N-2-1) -- (N-3-1) node[midway,left,pos=0.3] {$\pm1$};
        \draw (N-2-2) -- (N-3-1) node[midway,right,pos=0.3] {$\pm1$};
        \draw (N-2-4) -- (N-3-1) node[midway,right,pos=0.3] {$\pm1$};
        \node[below,DISTANCE] at (N-1-1) {$x_1$};
        \node[below,DISTANCE] at (N-1-2) {$x_2$};
        \node[below,DISTANCE] at (N-1-4) {$x_d$};
        \draw[<-] (-2,0.7) -- +(-0.5,0) node[left] {$w^{i}_{j,k}$};
        \draw[<-] (-4,2) -- +(-0.5,0) node[left] {$\sigma$};
        \draw[<-] (-4,5) -- +(-0.5,0) node[left] {$\sigma$};
        \node[right] at (6,5) {$M$ neurons};
        \node[right] at (6,2) {$M\times 2d$ neurons};
      \end{tikzpicture}
      \caption{\label{fig:neuralnetwork2layers}The function $\tilde f$ represented as a neural network.}
    \end{center}
  \end{figure}

  It remains to count the weights and biases of $\tilde f$ :
  \begin{itemize}
    \item the architecture has $M$ edges going to the output layer, due to the $\alpha_i$;
    \item it has $M$ biases associated to the neurons of the second hidden layer (they correspond to the terms $-J_i$);
    \item between the second and the first hidden layer, the architecture has $M\times2d$ edges (corresponding to the $\varepsilon_{i,j}$);
    \item it has $M\times2d$ biases associated to the neurons of the first hidden layer (the $\tilde b_j^{i}$);
    \item it has $M\times2d\times d$ edges between the first hidden layer and the entry (the $\tilde \w_j^{i}$).
  \end{itemize}
   Thus there are $2M+2M\times2d+M\times2d\times d=2(d^2+2d+1)M=2(d+1)^2M$ weights and biases in total.
\end{proof}

\subsubsection{Main developments of the proof of Proposition \ref{borne sup monotone}}\label{qerouneqvibet}
Let $N\in\mathbb N^*$ and $f \in \mathcal{M}^d$. In this section, we partition $[0,1)^d$ into cubes whose sizes depend on the maximal variation of $f$. Then we use this partition to construct a piecewise constant approximation $\tilde f$ of $f$; we will bound from above the $L^p(\lambda)$ approximation error $\|f-\tilde f\|_{L^p(\lambda)}$ by a function of $N$. This part is a direct reinterpretation of the proof of Proposition $3.1$ in \cite{GAO20071751}. We then apply  Proposition \ref{prop:2layers} to $\tilde f$ and obtain the announced result.

We first define some notation that will be used in the rest of the section, then we explain the algorithm used to divide $[0,1)^d$ into cubes. 
We fix the constant $K>1$ the following way:
\begin{equation*}\label{eq: def K}
  K:=\begin{cases}
       2^d \quad \text{    if }p=1,\\
       2^\beta  \quad \text{    otherwise, where }\beta=\frac12(d-1+\frac1{p-1}).
     \end{cases}
\end{equation*}
We also define an integer $l$ that corresponds to the number of cube decompositions:
\begin{equation}\label{l_def_qpeuon}
  l:=\left\lceil\frac{N\log 2}{\log K}\right\rceil = \begin{cases}
       \left\lceil\frac{N}{d}\right\rceil \quad \text{    if }p=1,\\
       \left\lceil\frac{N}{\beta}\right\rceil  \quad \text{    otherwise }.
     \end{cases}
\end{equation}
It is worth noting that this implies $K^{-l}\leq2^{-N}<K^{-l+1}$. 

Now we partition $[0,1)^d$ into dyadic cubes of the form $[a_1,b_1) \times \cdots \times [a_d,b_d)$. If $C$ is such a cube, we use the following convenient notation:
\begin{align*}
  \underline C&:=(a_1,\ldots,a_d)\in\mathbb R^d,&
  \overline C&:=(b_1,\ldots,b_d)\in\mathbb R^d,
\end{align*}
to refer to the smallest and largest vertices of $C$. The cube decompositions process reads as follow:

\begin{itemize}
  \item First we partition $[0,1)^d$ into $2^{Nd}$ cubes of side-length $2^{-N}$. We denote by $S_0$ the set of these cubes $C$ such that $f(\overline C)-f(\underline C)\leq K2^{-N}$ and by $R_0$ the set of the remaining cubes.
  \item For $1\leq i<l$, we partition each cube in the set $R_{i-1}$ (the remaining cubes at the step $i-1$) into $2^d$ cubes of equal size, and we denote by $S_i$ the set of obtained cubes $C$ of side-length $2^{-(i+N)}$ such that 
  \begin{equation}\label{eq: ineq partition cubes}
    f(\overline C)-f(\underline C)\leq K^{i+1}2^{-N}.
  \end{equation}\label{eq:defRi}
 Again, the set of remaining cubes is denoted by $R_i$.
  \item Lastly, we partition each cube in the set $R_{l-1}$ into $2^d$ cubes of equal size, and we denote by $S_l$ the set of obtained cubes of side-length $2^{-(l+N)}$.
\end{itemize}

Once the algorithm is done, each point in $[0,1)^d$ clearly belongs to one single cube of $\cup_{i=0}^l S_i$. For $i \in \{0,\ldots,l\}$, we let $\tilde{S}_i = \cup_{C \in S_i} C$.

We now define the piecewise constant approximation of $f$ by
  \begin{equation*}
    \forall x\in[0,1]^d,\quad \tilde f(x)=\sum_{C\in\bigcup_{0\leq i\leq l}S_i}f(\underline C)\mathds1_{x\in C},
  \end{equation*}
  where $\mathds1_{x\in C}$ denotes the indicator function of the cube $C$. We do not make the dependence explicit, but $\tilde f$ depends on the parameters $N$, $d$ and $p$. The number of cubes over which $\tilde f$ is constant is $\sum_{i=0}^l |S_i|$. This quantity is key when constructing the neural network according to Proposition \ref{prop:2layers}; in the next lemma, we bound from above $|S_i|$ for all $i= 0, \ldots, l$. Then, we will estimate the error $\|f-\tilde f\|_{L^p(\lambda)}$.
  
\begin{lemma}\label{lemma:Si}
  With the above notation:
  \begin{equation*}
    \forall i\in \{0,\ldots,l\} ,\quad |S_i| \leq 
    dK^{-i}2^{i(d-1)+Nd+1} 
  \end{equation*}
  Moreover,
    \begin{equation}\label{eq:measureSi}
    \lambda(\tilde{S}_i)\leq \left\{\begin{array}{ll}
    1 & \mbox{ if }i=0, \\ 
    2d(2K)^{-i}& \mbox{, otherwise}.
    \end{array}\right.
  \end{equation}
\end{lemma}

\begin{proof}
  By construction, we have
  \begin{equation*}
    \forall i\in\{1,\ldots,l\},\quad |S_i| + |R_i| =2^d |R_{i-1}|,
  \end{equation*}
since the set $S_i\cup R_i$ contains all the cubes of side-length $2^{-(i+N)}$, that have been constructed from the cubes of $R_{i-1}$. In particular,
\begin{equation}\label{eq:Si}
  \forall i\in\{1,\ldots,l\},\quad |S_i| \leq2^d |R_{i-1}|.
\end{equation}
It remains to bound $|R_{i-1}|$ from above for $i\geq1$. Define $V:=\{\underline C\colon C\in R_{i-1}\}$ the set of the smallest vertices of the cubes in $R_{i-1}$. We consider the classes of these vertices under the ``laying on the same extended diagonal'' equivalence relation. Since the cubes have side-length $2^{-(i-1+N)}$, there are less than $d2^{(i-1+N)(d-1)}$ equivalence classes. According to the pigeonhole principle, the largest class has at least $\left\lceil\frac{|V|}{d2^{(i-1+N)(d-1)}}\right\rceil$ elements; let us refer to this class as $\mathcal{D}$. Let $(C_j)_{1\leq j\leq J}$ be the set of cubes in $R_{i-1}$ having a point in $\mathcal{D}$ as lowest vertex. Since $f$ is non-decreasing and according to \eqref{eq: ineq partition cubes}, we have:
\begin{align*}
  1&\geq f(1,\ldots,1)-f(0,\ldots,0)\geq\sum_{j=1}^Jf(\overline{C_j})-f(\underline{C_j})\geq JK^{i}2^{-N}\\
  &\geq\frac{|V|}{d2^{(i-1+N)(d-1)}} K^{i}2^{-N}=\frac{|R_{i-1}|}{d2^{(i-1+N)(d-1)}} K^{i}2^{-N}.
\end{align*}
Thus
\begin{equation*}
  |R_{i-1}|\leq d2^{i(d-1)+Nd+1-d} K^{-i}.
\end{equation*}
The first statement of Lemma \ref{lemma:Si} follows from \eqref{eq:Si}.

For $i=0$, $\lambda(\tilde{S}_0) \leq 1$. For $1\leq i\leq l$, using the first statement of this lemma, we bound from above the measure of $\tilde{S}_i$:
  \begin{eqnarray*}
    \lambda(\tilde{S}_i)=\left(2^{-(i+N)}\right)^d |S_i| & \leq&  dK^{-i}2^{i(d-1)+Nd+1}2^{-d(i+N)}, \\
    & = & 2d(2K)^{-i}.
  \end{eqnarray*}
\end{proof}

  To show that $\tilde f$ is close to $f$ in $L^p(\lambda)$ norm, let us use the fact that $(\tilde S_i)_{0\leq i\leq l}$ is a partition of $[0,1)^d$ and decompose the error in three parts:
    \begin{equation*}
    \Vert f-\tilde f\Vert_{L^p(\lambda)}^p=\int_{\tilde{S}_0}\vert f(x)-\tilde f(x)\vert^p \dd x +\sum_{i=1}^{l-1}\int_{\tilde{S}_i}\vert f(x)-\tilde f(x)\vert^p \dd x+\int_{\tilde{S}_l}\vert f(x)-\tilde f(x)\vert^p \dd x.
  \end{equation*}
  In the next lemma, we control each term of the above sum  to bound from above $\Vert f-\tilde f\Vert_{L^p(\lambda)}$ by a function of $N$ that is independent of $f$ and tends to $0$ when $N$ tends to~$+\infty$.

 \begin{lemma}\label{lemma8}
    For any $1 \leq p < +\infty$, there exists a constant $c_{d,p} >0$ depending only on $d$ and $p$ such that for all $N\in\NN^*$
  \begin{equation}\label{majoref-ft_en_N}
    \Vert f-\tilde f\Vert_{L^p(\lambda)} \leq c_{d,p}\left\{\begin{array}{ll}
    2^{-N} & \mbox{ if } p(d-1) <d, \\
    2^{-N\frac{(1+1/\beta)}{p}} & \mbox{ if }p(d-1) >d, \\
    N^{\frac{1}{p}} 2^{-N} & \mbox{ if }p(d-1) =d, \\
    \end{array}\right.
  \end{equation}
  where $\tilde f$ is the function constructed for the parameters $N$, $d$ and $p$.
 \end{lemma}
 \begin{proof}

  For $0\leq i<l$, on any cube $C\in S_i$, we have
  \begin{equation}\label{eq:boundf}
    \forall x\in C,\quad\vert f(x)-\tilde f(x)\vert=\vert f(x)-f(\underline C)\vert\leq f(\overline C)-f(\underline C)\leq K^{i+1}2^{-N},
  \end{equation}
  since $f$ is non-decreasing, and by definition of $\tilde f$ and $S_i$.
  \begin{itemize}
    \item Using the fact that $\lambda(\tilde{S}_0)\leq1$ and by \eqref{eq:boundf}:
    \begin{equation}
      \label{eq:int1}
      \int_{\tilde{S}_0}\vert f(x)-\tilde f(x)\vert^p \dd x \leq(2^{-N}K)^p.
    \end{equation}
    \item Using \eqref{eq:measureSi} and \eqref{eq:boundf}, we get for all $i \in \{1,\ldots,l-1\}$
    \begin{equation}
      \label{eq:int2}
      \int_{\tilde{S}_i}\vert f(x)-\tilde f(x)\vert^p \dd x \leq(K^{i+1}2^{-N})^p2d(2K)^{-i}.
    \end{equation}
    \item On any $C\in S_l$, we have, for all $x\in C$, $\vert f(x)-\tilde f(x)\vert\leq\vert f(x)-f(\underline C)\vert\leq1$, and we get, using \eqref{eq:measureSi}:
    \begin{equation}
      \label{eq:int3}
      \int_{\tilde{S}_l}\vert f(x)-\tilde f(x)\vert^p \dd x\leq2d(2K)^{-l}.
    \end{equation}
  \end{itemize}
  Combining \eqref{eq:int1}, \eqref{eq:int2} and \eqref{eq:int3} we get:
  \begin{align}
    \Vert f-\tilde f\Vert_{L^p(\lambda)}^p&\leq(2^{-N}K)^p+\sum_{i=1}^{l-1}(K^{i+1}2^{-N})^p2d(2K)^{-i}+2d(2K)^{-l} \nonumber \\
    &\leq(2^{-N}K)^p+2^{1-Np}K^pd\sum_{i=1}^{l-1}\left(\dfrac{K^{p-1}}{2}\right)^i+2d(2K)^{-l}. \label{eq: ineqf-ftilde step 1}
  \end{align}
  It remains to bound the right-hand side of \eqref{eq: ineqf-ftilde step 1}, depending on the value of $p$ and $d$. Note that the behavior of this term depends on whether $\frac{K^{p-1}}{2}$ is larger or smaller than $1$.
  \begin{itemize}
    \item 
  Suppose that $p(d-1)<d$. In this case, we can have $p=1$ or $p>1$. If $p=1$, we have $\frac{K^{p-1}}{2}=\frac12<1$ and $\frac{1}{2K} < K^{-p}$. If $p>1$, we have:
  \begin{equation*}
    p(d-1)<d\ \Longleftrightarrow \ dp-p-d +1 < 1  \ \Longleftrightarrow \  d-1<\frac1{p-1}.
  \end{equation*}
  Thus, $\beta$ being the arithmetic mean of $d-1$ and $\frac1{p-1}$, we have $d-1<\beta<\frac1{p-1}$. Then $K=2^\beta<2^{1\slash (p-1)}$ and hence $\frac{K^{p-1}}{2}<1$ and $\frac1{2K}<K^{-p}$. Therefore, both for $p=1$ and $p>1$,
  \begin{align*}
      \sum_{i=1}^{l-1} \left( \frac{K^{p-1}}{2} \right)^i \leq \frac{K^{p-1}}{2-K^{p-1}} \quad \text{ and } \quad (2K)^{-l} \leq K^{-pl}.
  \end{align*}
  Since $K^{-l}\leq2^{-N}$, this leads to
  \begin{align}
    \Vert f-\tilde f\Vert_{L^p(\lambda)}^p&\leq(2^{-N}K)^p+2^{1-Np}K^pd\dfrac{K^{p-1}}{2-K^{p-1}}+2dK^{-pl} \nonumber \\
    &\leq\left(K^p+2K^pd\dfrac{K^{p-1}}{2-K^{p-1}}+2d\right)2^{-Np}.\nonumber
  \end{align}
  We thus have, setting $c_1:=\left(K^p+2K^pd\dfrac{K^{p-1}}{2-K^{p-1}}+2d\right)^{\frac1p}$,
  \begin{equation*}
    \Vert f-\tilde f\Vert_{L^p(\lambda)} \leq c_1 2^{-N}.
  \end{equation*}
Notice $c_1$ only depends on $d$ and $p$.
  \item Suppose that $p(d-1)>d$. We have $p>1$ and  $d-1>\beta>\frac1{p-1}$. Then $K=2^\beta>2^{1\slash(p-1)}$ and hence $\frac{K^{p-1}}{2}>1$, which entails using \eqref{eq: ineqf-ftilde step 1}
  \begin{align*}
    \Vert f-\tilde f\Vert_{L^p(\lambda)}^p&\leq(2^{-N}K)^p+2^{1-Np}K^pd\dfrac{(K^{p-1}\slash2)^l}{K^{p-1}\slash2-1}+2d(2K)^{-l} \\
    &\leq2^{-Np}K^p+2^{-Np}K^{pl}\frac{2K^pd}{K^{p-1}\slash2-1}(2K)^{-l}+2d(2K)^{-l}.
  \end{align*}
  
  Since $p>1+\frac{1}{\beta}$, we have $2^{-Np} \leq 2^{-N(1+\frac{1}{\beta})}$.
  Also, since $K=2^{\beta}$, $(2K)^{-l} = 2^{-l(\beta+1)}$, and since $l \geq \frac{N \log(2)}{\log(K)} = \frac{N}{\beta}$, we have $(2K)^{-l} \leq 2^{-\frac{N}{\beta}(\beta+1)} = 2^{-N(1+\frac{1}{\beta})}$.
  Finally, since $2^{-N}K^l < K$,
  \begin{align*}
      \Vert f-\tilde f\Vert_{L^p(\lambda)}^p&\leq \left(K^p+K^p\frac{2K^{p}d}{K^{p-1}\slash2-1}+2d\right)2^{-N(1+1\slash\beta)}
  \end{align*}
  
  We thus have, setting $c_2:=\left(K^p+\frac{2K^{2p}d}{K^{p-1}\slash2-1}+2d\right)^{\frac1p}$,  \begin{equation*}
    \Vert f-\tilde f\Vert_{L^p(\lambda)}\leq c_2 2^{-\frac{N(1+1\slash\beta)}p}.
  \end{equation*}
Notice $c_2$ only depends on $d$ and $p$.

  \item Suppose that $p(d-1)=d$. It implies $p>1$ and $p-1=\frac1{d-1}$, then $\beta=d-1$. We thus have $K^{p-1}=2^{(d-1)(p-1)}=2$. Therefore, \eqref{eq: ineqf-ftilde step 1} becomes
  \begin{equation*}
    \Vert f-\tilde f\Vert_{L^p(\lambda)}^p\leq2^{-Np}K^p+2^{-Np}2K^pd(l-1)+2d(K^p)^{-l}.
  \end{equation*}
  On the one hand, we have $K^{-l}\leq2^{-N}$. On the other, we have $2^{-N}<K^{-l+1}$, so $l-1<N\frac{\log 2}{\log K}=\frac{N}{d-1}$. Putting it all together, we get
  \begin{align*}
    \Vert f-\tilde f\Vert_{L^p(\lambda)}^p&\leq2^{-Np}K^p+2^{-Np}2K^pd(l-1)+2d2^{-Np}\\
    &\leq\left(K^p+2K^p\frac d{d-1}+2d\right)N2^{-Np}.
  \end{align*}
  We thus have, setting $c_3:=\left(K^p+2K^p\frac d{d-1}+2d\right)^\frac1p$,
  \begin{equation*}
    \Vert f-\tilde f\Vert_{L^p(\lambda)}\leq c_3 N^\frac1p2^{-N}.
  \end{equation*}
  Notice $c_3$ only depends on $d$ and $p$.

  \end{itemize}
  Letting $c_{d,p} = \max\{c_1,c_2,c_3\}$ yields the result.
   \end{proof}
 According to Proposition~\ref{prop:2layers}, the function $\tilde f$ constructed for a given $N\in\NN^*$ can be implemented by a Heaviside neural network with two hidden layers and $W=2(d+1)^2\sum_{i=0}^l|S_i|$ weights.
  Using Lemma \ref{lemma:Si}, we obtain
  \begin{align*}
    W= 2(d+1)^2\sum_{i=0}^l|S_i|&\leq2(d+1)^2\sum_{i=0}^ldK^{-i}2^{i(d-1)+Nd+1}\\
    &= 2^{Nd+2}d(d+1)^2\sum_{i=0}^l\left(\frac{2^{d-1}}K\right)^i.
  \end{align*}
  We let, for all $N\in\NN^*$,
  \begin{equation}\label{WN_eq}
  W_N := 2^{Nd+2}d(d+1)^2\sum_{i=0}^l\left(\frac{2^{d-1}}K\right)^i.
  \end{equation}
 Although we do not make the dependence explicit, $W_N$ also depends on $d$ and $p$. Observe that for all $d\geq 1$: $(W_N)_{N\in\NN^*}$ is non-decreasing and $\lim_{N\rightarrow +\infty} W_N = + \infty$.
 
 \begin{lemma}\label{lastLemma}
 With the above notation: For any $+\infty >p\geq 1$, there exist constants $W'_{\min}, c'_{d,p} >0$ depending only on $d$ and $p\geq 1$ such that for all $N$ satisfying $W_N\geq W'_{\min}$
  \[\Vert f-\tilde f\Vert_{L^p(\lambda)} \leq c'_{d,p} ~g(W_{N+1})
  \]
  where $\tilde f$ is constructed for the parameters $N$, $p$ and $d$, and where for all $W\geq 1$,
  \[g(W) = 
  \begin{cases}
      W^{-1/d}&\text{if $(d-1)p<d$,}\\
      W^{-\frac1{p(d-1)}}&\text{if $(d-1)p>d$,}\\
      W^{-1/d}\log W&\text{if $(d-1)p=d$.}
    \end{cases}
  \]
  \end{lemma}
 \begin{proof}
  Again, we distinguish three cases depending on the values of $p$ and $d$.
  \begin{itemize}
    \item Suppose that $p(d-1)< d$: if $p=1$, $\frac{2^{d-1}}K=\frac12<1$; if $p>1$, since $\frac{1}{p-1} > d-1$, $\beta>d-1$ and $\frac{2^{d-1}}K=2^{d-1-\beta}<1$. Thus, in both cases $\frac{2^{d-1}}{K}<1$ and for all $N\geq 1$,
    \begin{equation*}
      W_N\leq2^{Nd}\left(\frac{4d(d+1)^2}{1-2^{d-1-\beta}}\right)=:2^{Nd}c''_{d,p}. 
    \end{equation*}
    Writing the inequality for $N+1$, we obtain
     \begin{equation*}
      W_{N+1}\leq 2^{Nd} 2^d c''_{d,p} \;. 
    \end{equation*}
   
   That is: $2^{-N} \leq 2 \left(\frac{c''_{d,p}}{W_{N+1}}\right)^{1\slash d}$. Combined with \eqref{majoref-ft_en_N}, this provides
    \begin{equation*}
      \Vert f-\tilde f\Vert_{L^p(\lambda)} \leq 
      2 c_{d,p}\left(\frac{c''_{d,p}}{W_{N+1}}\right)^{1\slash d}= d_{d,p} W_{N+1}^{-1\slash d},
    \end{equation*}
    for $d_{d,p}= 2 c_{d,p} (c''_{d,p})^{1\slash d}$ and all $N\in\NN^*$.
    \item If $p(d-1)> d$, then $\beta<d-1$ and $\frac{2^{d-1}}K = 2^{d-1-\beta}>1$. Thus, reminding the definition of $l$ in \eqref{l_def_qpeuon}, we have for all $N\geq 1$
    \begin{align*}
      W_N&\leq2^{Nd}2^{(d-1-\beta)(l+1)}\left(\frac{4d(d+1)^2}{2^{d-1-\beta}-1}\right)\leq2^{Nd}2^{(d-1-\beta)(N\slash\beta+2)}\left(\frac{4d(d+1)^2}{2^{d-1-\beta}-1}\right)\\
      &=
      2^{N(d+(d-1)/\beta-1)}\left(\frac{4d(d+1)^22^{2(d-1-\beta)}}{2^{d-1-\beta}-1}\right)=:2^{N(1+\frac{1}{\beta})(d-1)}c''_{d,p},
    \end{align*}
    for a different constant $c''_{d,p}$.
    Writing again this inequality for $N+1$, we obtain
    \[W_{N+1} \leq c''_{d,p} 2^{(1+\frac{1}{\beta})(d-1)}~ 2^{N(1+\frac{1}{\beta})(d-1)},
    \]
    which we can write $2^{-N(1+\frac{1}{\beta})} \leq 2^{(1+\frac{1}{\beta})} \left( \frac{c''_{d,p}}{W_{N+1}}\right)^{\frac{1}{d-1}}$. This provides
    $$2^{-N\frac{(1+1/\beta)}{p}} \leq 2^{\frac{(1+1/\beta)}{p}} \left( \frac{c''_{d,p}}{W_{N+1}}\right)^{\frac{1}{p(d-1)}}.$$
    Therefore, using \eqref{majoref-ft_en_N}, we obtain
    \begin{equation*}
      \Vert f-\tilde f\Vert_{L^p(\lambda)}\leq
      c_{d,p} 2^{\frac{(1+1/\beta)}{p}} \left(\frac{c''_{d,p}}{W_{N+1}}\right)^{\frac{1}{p(d-1)}}= d'_{d,p} W_{N+1}^{-\frac{1}{p(d-1)}},
    \end{equation*}
    for $d'_{d,p} =c_{d,p}~ 2^{\frac{(1+1/\beta)}{p}}~  (c''_{d,p})^{\frac{1}{p(d-1)}} $ and all $N\in\NN^*$.
    \item If $p (d-1)= d$, then $\beta=d-1$ and  $\frac{2^{d-1}}K=1$. Thus, reminding the definition of $l$ in \eqref{l_def_qpeuon}, we have for all $N\geq 1$
    \begin{align}
      W_N&= 2^{Nd+2}d(d+1)^2(l+1)\leq2^{Nd+2}d(d+1)^2\left(\frac N\beta+2\right) \nonumber \\
      &=2^{Nd}\left(\frac N\beta+2\right)\left(4d(d+1)^2\right)=:2^{Nd}\left(\frac N{d-1}+2\right)c''_{d,p} \nonumber \\
      &\leq2^{d(d-1)\left(\frac N{d-1}+2\right)}\left(\frac N{d-1}+2\right)c''_{d,p} \nonumber \\
      & = \exp\left(d(d-1)\left(\frac N{d-1}+2\right)\log 2\right)\left(\frac N{d-1}+2\right)c''_{d,p}\label{eq:weights3}
    \end{align}
    where $c''_{d,p} = 4d(d+1)^2$. Setting
    \[\tilde W_N:=\frac{d(d-1)W_N\log2}{c''_{d,p}} \qquad \mbox{and}\qquad
      \tilde N:=d(d-1)\left(\frac N{d-1}+2\right)\log 2,
    \]
    we can rewrite \eqref{eq:weights3} as:
   \begin{equation}
    \label{tmpnvet}
     \tilde W_N\leq \tilde N\exp(\tilde N) \;.
   \end{equation}
Since $d\geq 2$, $c''_{d,p}>0$, $(W_N)_{N\in\NN^*}$ is non-decreasing and $\lim_{N\rightarrow +\infty} W_N = +\infty$, there exists $W'_{min}$  such  that, for all $N$ satisfying $W_N\geq W'_{min}$, we have the following:
\begin{equation}\label{lesEqs}
\left\{\begin{array}{l}
\log(\tilde{W}_N)>1 \\
\log(\tilde{W}_{N+1})>2~\log(2)~d(d-1) \\
\frac{\log(\tilde W_{N+1})}{d \log(2)} - \frac{\log \log (\tilde W_{N+1})}{d \log(2)} -2(d-1) > \frac{1}{p\log(2)}\\
\log W_{N+1} \geq \log \left(\frac{d(d-1)\log2}{c''_{d,p}}\right).
\end{array}\right.
\end{equation}
These inequalities will be used latter in the proof and, from now on, we always consider $N$ such that $W_N\geq W'_{min}$.

Let us first show by contradiction that, for all $N$ satisfying $W_N\geq W'_{min}$, \eqref{tmpnvet} implies that
   \begin{equation}\label{eronbetb}
     \tilde N\geq\log \tilde W_N-\log\log \tilde W_N.
   \end{equation}
Indeed, if the latter does not hold
\[\tilde N < \log \tilde W_N-\log\log \tilde W_N,
\]
\[\exp(\tilde N) < \frac{ \tilde W_N} {\log \tilde W_N},
\]
  and therefore, multiplying the two inequalities, since \eqref{lesEqs} implies that $\tilde W_N >0$, $\log \tilde W_N >0$ and  $\log(\log (\tilde W_N)) > 0$,
\[\tilde N \exp(\tilde N)< \tilde W_N.
\]  
The latter being in contradiction with \eqref{tmpnvet}, we have proved that, for all $N$ satisfying $W_N\geq W'_{min}$, \eqref{eronbetb} holds. Using the definition of $\tilde N$, we deduce 
   \begin{eqnarray*}
     N & \geq& \left(\frac{\log\tilde W_N-\log\log \tilde W_N}{d(d-1)\log(2)}-2\right)(d-1) \\
     & = &  \frac{\log(\tilde W_N)}{d \log(2)} - \frac{\log \log \tilde W_N}{d \log(2)} + c,
   \end{eqnarray*}
   for the constant $c= -2(d-1) < 0$. Since $(W_N)_{N\in\NN}$ is non-decreasing, for all $N$ satisfying $W_N\geq W'_{min}$, $W_{N+1} \geq W'_{min}$ and the inequality also holds for $N+1$. That is
   \begin{equation}\label{jerbutb}
       N+1 \geq \frac{\log(\tilde W_{N+1})}{d \log(2)} - \frac{\log \log \tilde W_{N+1}}{d \log(2)} + c.
   \end{equation}
   Using \eqref{majoref-ft_en_N}, we obtain:
  \[\Vert f-\tilde f\Vert_{L^p(\lambda)}\leq c_{d,p} N^{\frac{1}{p}} 2^{-N} \leq 2 c_{d,p} (N+1)^{\frac{1}{p}}2^{-(N+1)}.
  \]
  Since, for $t>\frac{1}{p\log(2)}$, the function $t\longmapsto t^{\frac{1}{p}}2^{-t}$ is non-increasing, using \eqref{jerbutb} and \eqref{lesEqs} and the fact that $- \frac{\log \log \tilde W_{N+1}}{d \log(2)} + c <0$, we obtain
  \begin{eqnarray*}\Vert f-\tilde f\Vert_{L^p(\lambda)} & \leq &  2 c_{d,p} \left(\frac{\log(\tilde W_{N+1})}{d \log(2)}\right)^{\frac{1}{p}}2^{-\frac{\log(\tilde W_{N+1})}{d \log(2)}} 2^{\frac{\log \log \tilde W_{N+1}}{d \log(2)}}2^{-c}, \\
  & = & \left( \frac{2^{1-c} c_{d,p}}{(d\log(2))^{1/p}}\right) ~ (\log \tilde W_{N+1})^{\frac{1}{p}+\frac{1}{d}} ~ \tilde W_{N+1}^{-\frac{1}{d}}\\
  & = & \left( \frac{2^{1-c} c_{d,p}}{(d\log(2))^{1/p}}\right) ~  \tilde W_{N+1}^{-\frac{1}{d}} ~ \log \tilde W_{N+1}, \\
  \end{eqnarray*}
  since $p(d-1) = d$ implies $\frac{1}{p}+\frac{1}{d} = 1$. Finally, using the definition of $\tilde W_{N}$ and \eqref{lesEqs}, we obtain
  \[\Vert f-\tilde f\Vert_{L^p(\lambda)}  \leq d''_{d,p} W_{N+1}^{-\frac{1}{d}} ~ \log W_{N+1},
  \]
 for the  constant $d''_{d,p} = 2 \left( \frac{2^{1-c} c_{d,p}}{(d\log(2))^{1/p}}\right) \left(\frac{d(d-1)\log2}{c''_{d,p}}\right)^{-1/d}$ and all $N\in\NN^*$ such that $W_N \geq W'_{min}$. Notice $d''_{d,p}$ only depends on $d$ and $p$.
  \end{itemize}
  Taking $c'_{d,p} = \max(d_{d,p},d'_{d,p},d''_{d,p})$ provides the announced statement.
\end{proof}

  \paragraph{Proof of Proposition \ref{borne sup monotone}.}
  Take $W_{\min} = \max(W'_{\min}, W_1)$ and $c=c'_{d,p}$, where $W'_{\min}$ and $c'_{d,p}$ are from Lemma \ref{lastLemma} and $W_1$ is defined in \eqref{WN_eq}. Let $W\geq W_{\min}$, there exists $N\in\NN^*$ such that
    \[ W_N\leq W < W_{N+1}.
  \]
  Consider the architecture $\mathcal{A}$ with $W$ weights, as in Proposition \ref{prop:2layers}, which allows to represent piecewise-constant functions with less than $\frac{W}{2(d+1)^2}$ cubic pieces. It can represent piecewise-constant functions with $\frac{W_N}{2(d+1)^2}$ pieces.
   
  For any $f\in\mathcal{M}^d$, the function $\tilde f$ obtained for the parameter $N$ is a piecewise-constant function with at most $\frac{W_N}{2(d+1)^2}$ pieces, therefore we have $\tilde f\in H_{\mathcal A}$ and, according to Lemma \ref{lastLemma}, $\tilde f$ satisfies
  \[ \Vert f-\tilde f\Vert_{L^p(\lambda)}  \leq c'_{d,p}g(W_{N+1}).
  \]
  Moreover, since $g$ is non-increasing, we have using $c=c'_{d,p}$
  \[\Vert f-\tilde f\Vert_{L^p(\lambda)}  \leq c \  g(W).
  \]
  Therefore, for any $f\in\mathcal{M}^d$,
  \[\inf_{g\in H_{\mathcal{A}}}\Vert f-g\Vert_{L^p(\lambda)} \leq c \ g(W)
  \]
  and so does the supremum over $f$ in $\mathcal{M}^d$.
  
  This concludes the proof of Proposition \ref{borne sup monotone}.

\subsection{Proof of Proposition \ref{borne inf norme infinie}}\label{preuve borne inf norme infinie}

 \textbf{Step 1: we prove the result in dimension $d=2$.}

 We consider the closed disk of radius $1$, centered at $(1,1)$,
  \[
  \mathcal{C} = \left\{x \in \R^2: \sum_{i=1}^2 (x_i-1)^2 \leq 1\right\} \;.
  \]
The intersection between $(0,1)^2$ and the topological boundary $\partial \mathcal{C}$
 of $\mathcal{C}$ is the quarter of circle:
   \[
  \partial \mathcal{C}\cap (0,1)^2 = \left\{x \in (0,1)^2: \sum_{i=1}^2 (x_i-1)^2 = 1\right\} \;.
  \]
 
 We denote by $f : [0,1]^2 \rightarrow \{0,1\}$ the indicator function of the set $\mathcal{C}\cap [0,1]^2$. The set $\mathcal{C}\cap [0,1]^2$, the set $\partial \mathcal{C}\cap (0,1)^2$ and the function $f$ are represented on Figure \ref{fig:counterex2}. 
 
  \begin{figure}
  \begin{center}
   \begin{tikzpicture}[y=1cm,scale=5]
      \draw (0,0) rectangle (1,1);
      \begin{scope}
        \path[clip] (0,1) arc (180:270:1) -- (1,1) -- cycle;
        \foreach \x in {0,0.08,...,2} {
          \draw[thin,densely dotted] (\x,0) ++(-1,0) -- +(1,1);
        }
      \end{scope}
      \draw[thick] (0,1) arc (180:270:1);
      \node[fill=white,inner sep=5pt] at (0.15,0.15) {\scalebox20};
      \node[fill=white,inner sep=5pt] at (0.6,0.6) {\scalebox21};
      \node[left] at (0,0) {0};
      \node[left] at (0,1) {1};
      \node[right] at (1,0) {1};
      \node[right] at (1,1){(1,1)};
    \end{tikzpicture}
    \caption{\label{fig:counterex2}The set $\mathcal C$, the set $\partial \mathcal{C}\cap (0,1)^2$ and the indicator function $f$.}
  \end{center}
\end{figure}
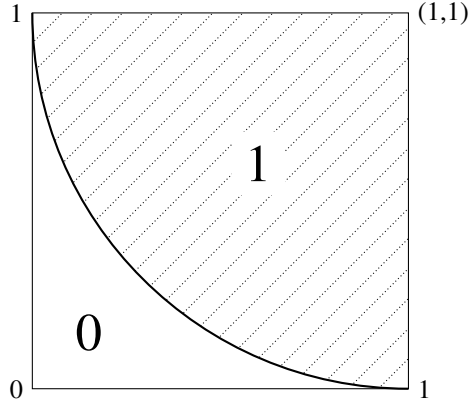

  Since no point in $\mathcal{C}^c \cap [0,1]^2$ has all its coordinates strictly larger than those of a point in $\mathcal{C}$, we have $f\in\mathcal{M}^2$ (monotonic functions of $2$ variables). We consider an arbitrary neural network architecture $\mathcal{A}$ and $g \in H_{\mathcal{A}}$. 
  
  Let $W\geq1$ be the number of weights in the architecture $\mathcal{A}$. As is well known for Heaviside neural networks, there exist $K\in\NN$ with $K\leq2^W$, reals $\alpha_j$ and polygons $A_j\subset [0,1]^2$, for $j\in\{1, \ldots,K\}$, such that for all $x\in[0,1]^2$ $$g(x)=\sum_{j=1}^{K}\alpha_j\mathds1_{A_j}(x).$$ 
  Moreover, $(A_j)_{1\leq j \leq K}$ form a partition of $[0,1]^2$. 
  
  The proof relies on the fact (proved afterwards) that, if $\Vert f-g\Vert_{\infty}<\frac12$ then $\partial\mathcal C\cap(0,1)^2$ is finite. The latter being false, we conclude that $\Vert f-g\Vert_{\infty} \geq \frac12$.
  
  Assume from now on that $\Vert f-g\Vert_{\infty}<\frac12$.
  This implies that $g>\frac12$ on $\mathcal C$, and $g<\frac12$ elsewhere.
  Let us first show that we then have 
  $$\partial\mathcal C\cap(0,1)^2 \ \subset \ \bigcup_{j=1}^K \partial A_j.$$ 
  
  Indeed, if the latter were not true, then there would exist $x\in\partial\mathcal C\cap(0,1)^2$ and $j\in\{1,\ldots,K\}$ such that $x\in\mathring A_j$. Since $\mathcal C$ is closed, $x\in \mathcal C$. Let $\epsilon>0$ be such that $B(x,\epsilon)\subset \mathring A_j$. We have $B(x,\epsilon)\not\subset\mathcal C$ (otherwise, $x$ belongs to the interior of $\mathcal C$ which contradicts $x\in\partial\mathcal C$). Thus there exists $z\in B(x,\epsilon)\setminus\mathcal C$. Since $g>\frac12$ on $\mathcal C$, and $g<\frac12$ elsewhere, we have $$g(z)<\frac12<g(x).$$
  This is not possible since $x,z\in\mathring A_j$ and $g$ is constant on $A_j$. This concludes the proof of the following fact: if $\Vert f-g\Vert_{\infty}<\frac12$ then $\partial\mathcal C\cap(0,1)^2 \subset\bigcup_{1\leq j\leq K}\partial A_j$.
  
  Since the $A_j$ are polygons (recall that we work in dimension $2$), their boundaries are finite unions of closed line segments. Then $\partial\mathcal C\cap(0,1)^2$ is included in a finite union of closed line segments which we denote $S_m$, for $m\in\{1,\ldots,M\}$. The reader may already see that this is in contradiction with the fact that $\partial\mathcal C\cap(0,1)^2$ is a quarter circle. To detail this argument and complete the announced proof, we show that $\partial\mathcal C\cap(0,1)^2 \subset \bigcup_{m=1}^M S_m$  implies that 
  $\partial\mathcal C\cap(0,1)^2$ is finite.
  
  To do so, since when $\partial\mathcal C\cap(0,1)^2 \subset \bigcup_{m=1}^M S_m$ we have
  \[ \bigcup_{m=1}^M \left( \partial\mathcal C\cap (0,1)^2 \cap S_m \right) = \partial\mathcal C\cap (0,1)^2,
  \]
  it suffices to prove that the intersection of  any closed line segment $S$ with $\partial\mathcal C\cap(0,1)^2$ contains at most $2$ points. 
  
  Denote by $S$ a closed line segment: $\mathcal C$ and $S$ are convex and hence connected, thus $\mathcal C\cap S$ is either empty, a singleton or a line segment, as a connected compact subset of $S$. If it is empty, then \textit{a fortiori}, $\partial\mathcal C\cap(0,1)^2\cap S=\emptyset$. If it is not, denote by $y$ and $z$ its extremities (assuming $z=y$ in the case of a singleton). By strict convexity of the function $x \mapsto \sum_{i=1}^2 (x_i-1)^2$, the open line segment $(y,z)$ is included in $\mathring{\mathcal C}$ ( $(y,z)=\emptyset$ in the case of a singleton), hence 
  $$\partial\mathcal C\cap(0,1)^2\cap S\ \subset\  [y,z] \setminus\mathring{\mathcal C}  \ \subset \ \{y,z\}.$$ 
  In any case, we have $|\partial\mathcal C\cap(0,1)^2\cap S|\leq2$.
  
  This concludes the proof of the fact: if $\Vert f-g\Vert_{\infty}<\frac12$ then $\partial\mathcal C\cap(0,1)^2 $ is finite and concludes the proof in the case $d=2$.

\textbf{Step 2: we prove the result in any dimension $d \geq 2$,} by a reduction to dimension $2$.

We define
  \[
  \mathcal{C} = \left\{x \in \R^d: \sum_{i=1}^d (x_i-1)^2 \leq 1\right\} \;,
  \]
  and the function $f:[0,1]^d \to \R$ by
  \[
  f(x_1,\ldots,x_d) = \mathds{1}_{(x_1,\ldots,x_d) \in \mathcal{C}} \;.
  \]
  Consider an arbitrary neural network architecture $\mathcal{A}$ and $g \in H_{\mathcal{A}}$. That is, $g$ can be represented by a Heaviside neural network with $d$ input neurons. Note that
  \begin{align*}
  & \sup_{x_1,x_2,x_3\ldots,x_d \in [0,1]} |f(x_1,x_2,x_3\ldots,x_d) - g(x_1,x_2,x_3\ldots,x_d)| \\
  & \qquad \qquad \geq \sup_{x_1,x_2 \in [0,1]} |f(x_1,x_2,1\ldots,1) - g(x_1,x_2,1\ldots,1)| \\
  & \qquad \qquad \geq \frac{1}{2} \;,
  \end{align*}
  where the last inequality is by the result of Step~1, since $(x_1,x_2) \in [0,1]^2 \mapsto f(x_1,x_2,1\ldots,1)$ is the indicator function of Step~1, and $(x_1,x_2) \in [0,1]^2 \mapsto g(x_1,x_2,1\ldots,1)$ can be represented by a Heaviside neural network with $2$ input neurons. This concludes the proof.

\textbf{Remark.} Note from the above proof that, though we only stated the impossibility result for piecewise-constant activation functions, an analogue statement in fact holds more generally for piecewise-affine activation functions.

\section{Barron space}
\label{appendix_perspectives}

In Section~\ref{sec:Barron} we mentioned that the Barron space introduced in \cite{Barron} is one among several examples for which approximation theory provides ready-to-use lower bounds on the packing number. This space has received renewed attention recently in the deep learning community, in particular because its \say{size} is sufficiently small to avoid approximation rates depending exponentially on the input dimension $d$. Next we detail how to apply Corollary~\ref{main_result} in this case.

\paragraph{Definition of the Barron space.}
We start by introducing the Barron space, as defined in \cite{https://doi.org/10.48550/arxiv.2112.12555}. Let $d \in \NN^*$. For any constant $C > 0$, the Barron space $B_d(C)$ is the set of all functions $f~:~[0,1]^d~\rightarrow~[0,1]$ for which there exist a measurable function $F : \mathbb{R}^d \rightarrow \mathbb{C}$ and some $c \in [-C, C]$ such that, for all $x \in [0,1]^d$,
\begin{align*}
    f(x) = c + \int_{\mathbb{R}^d} (e^{i x\cdot \xi}- 1) F(\xi) \dd \xi \qquad \text{and} \qquad \int_{\mathbb{R}^d} \|\xi\|_2 |F(\xi)| \dd \xi \leq C,
\end{align*}
where $x\cdot \xi$ denotes the standard scalar product in between $x$ and $\xi$.

\paragraph{Known lower bound on the packing number.}
Petersen and Voigtlaender \cite{https://doi.org/10.48550/arxiv.2112.12555} showed a tight lower bound on the log packing number in $L^p(\lambda,[0,1]^d)$ norm, which we recall below.

\begin{prop}[Proposition 4.6 in \cite{https://doi.org/10.48550/arxiv.2112.12555}]
\label{prop:Barron-logpacking}
    Let $1 \leq p \leq +\infty$. There exist constants $\varepsilon_0, c_0 > 0$ depending only on $d$ and $C$ such that for any $\varepsilon \leq \varepsilon_0$,
    \begin{align}
        \label{eq:barron_logpacking}
        \log M(\varepsilon, B_d(C), \|\cdot\|_{L^p}) \geq c_0 \varepsilon^{-1/(\frac{1}{2} + \frac{1}{d})}.
    \end{align}
\end{prop}

\paragraph{Consequence on the approximation rate by piecewise-polynomial neural networks.}

Plugging the lower bound of Proposition~\ref{prop:Barron-logpacking} in Corollary~\ref{main_result}, we obtain the following lower bound on the approximation error of the Barron space by piecewise-polynomial neural networks.

\begin{prop}
Let $1 \leq p < +\infty$, $d \geq 1$. Let $\sigma : \mathbb{R} \rightarrow \mathbb{R}$ be a piecewise-polynomial function on $K \geq 2$ pieces, with maximal degree $\nu \in \mathbb{N}$. Consider the Barron space $B_d(C)$ defined above, with $C>0$. There exist positive constants $c_1, c_2, c_3, W_{\emph{min}}$ depending only on $d$, $p$, $C$, $K$ and $\nu$ such that, for any architecture $\mathcal{A}$ of depth $L \geq 1$ with $W\geq W_{\emph{min}}$ weights, and for the activation $\sigma$, the set $H_{\mathcal{A}}$ (cf. Section~\ref{s:intro}) satisfies
    \label{prop:lb_approximation_barron}
    \begin{align}
    \label{eq:lowerbound_barron}
        \sup_{f \in B_d(C)} \inf_{g \in H_{\mathcal{A}}} \|f - g\|_{L^p(\lambda)} \geq
        \left\{
        \begin{array}{l l}
            c_1  W^{-1-\frac{2}{d}} \log^{-1-\frac{2}{d}}(W) & \text{ if } \nu \geq 2 \,, \\
            c_2  (LW)^{-\frac{1}{2}-\frac{1}{d}} \log^{-\frac{3}{2}-\frac{3}{d}}(W)& \text{ if } \nu=1 \,, \\
            c_3  W^{-\frac{1}{2}-\frac{1}{d}} \log^{-\frac{3}{2}-\frac{3}{d}}(W)& \text{ if } \nu=0 \,.
        \end{array}
        \right.
    \end{align}
\end{prop}

\end{document}